\documentclass{article}

\PassOptionsToPackage{numbers}{natbib}

  \usepackage[preprint]{neurips_2026}

\usepackage[utf8]{inputenc} % allow utf-8 input
\usepackage[T1]{fontenc}    % use 8-bit T1 fonts
\usepackage{hyperref}       % hyperlinks
\usepackage{url}            % simple URL typesetting
\usepackage{booktabs}       % professional-quality tables
\usepackage{amsfonts}       % blackboard math symbols
\usepackage{nicefrac}       % compact symbols for 1/2, etc.
\usepackage{microtype}      % microtypography
\usepackage[table]{xcolor}         % colors

\usepackage{graphicx}
\usepackage{xspace}
\usepackage{subcaption}

\usepackage{amsmath}
\usepackage{amssymb}
\usepackage{mathtools}
\usepackage{amsthm}

\usepackage{algorithm}
\usepackage{algorithmic}
\usepackage{bbm}
\usepackage{multirow}
\usepackage{diagbox}
\usepackage{natbib}
\usepackage{wrapfig}
\usepackage{titletoc}
\theoremstyle{plain}
\newtheorem{theorem}{Theorem}[section]

\newtheorem{lemma}[theorem]{Lemma}

\theoremstyle{definition}

\theoremstyle{remark}

\newcommand{\alp}{WR-Offline\xspace}
\newcommand{\efalp}{WR-Online\xspace}

\newcommand{\optiroute}{WISERouter\xspace}

\definecolor{methodshade}{RGB}{239,239,250}

\newcommand{\eat}[1]{}

\title{\optiroute: LLM Routing with Workload Budget Constraint}

\author{%
  Yifei Li\thanks{Corresponding author: \texttt{yfli@cs.ubc.ca}} \quad
  Zihui Gao \quad
  Laks V.S. Lakshmanan \\
  The University of British Columbia
}

\begin{document}

\maketitle

\begin{abstract}
  Large language models (LLMs) achieve impressive performance across multiple domains, but using the most capable model for every query is prohibitive at scale. LLM routing exploits diversity in model capability and cost by assigning each query to a suitable model to balance utility and budget. Current methods have two limitations: (i) they either use heuristics that do not always enforce the budget constraint or impose a fixed \textit{per-query} budget that cannot adapt across the workload and leads to suboptimal performance;  (ii) they require supervised learning on a dense dataset with statistics for every query–model pair, which is expensive to collect. To address these challenges, we formulate LLM routing as a constrained contextual multi-armed bandit problem and introduce WISERouter (WR for short), a framework that supports offline learning from historical interactions as well as online learning with exploration. We further prove that WR-Online achieves a sublinear regret bound of $O(\sqrt{T})$ over a time horizon $T$. Empirical results on RouterBench and SWE-Bench demonstrate that (i) WR-Offline surpasses existing baselines in performance under a fixed budget and adheres more closely to budget constraints, and (ii) WR-Online achieves comparable performance to the baselines, while using substantially less exploration data. 
\end{abstract}

\section{Introduction}

While state-of-the-art large language models (LLMs) continue to advance their remarkable capabilities, their token-wise inference cost often makes them prohibitively expensive for high-throughput workloads. Leading AI services can charge \$100 per user per month~\cite{claude_team_pricing} under business subscription; at the API billing level, frontier models like GPT-5.5 cost up to 6.7$\times$ more per token than smaller models like GPT-5.4-mini~\cite{openai_api_pricing}, pushing organizations or businesses that process high-throughput workload queries -- such as customer support ticket resolution or document summarization -- to inference bills exceeding thousands of dollars per month. It has been found that smaller models can often serve as effective proxies, delivering comparable performance at a fraction of cost~\cite{hybridllm, routellm}. Furthermore, various benchmark leaderboards~\cite{chatbotArena24,livebench25} and empirical studies reveal that different LLMs exhibit varying strengths across different query types~\cite{mixLLM2025, HELM23}, making inference-time model selection attractive for both cost saving and quality improvement. 

Inference-time model selection has emerged as a practical approach to this challenge. The two dominant paradigms are LLM routing~\cite{hybridllm, routellm}, which assigns each query directly to the most suitable model, and cascading~\cite{frugalgpt2024, cascadeMixoT24}, which queries models sequentially and may invoke multiple models per query until a quality threshold is met. 
Recent work has validated inference-time LLM selection across diverse benchmarks~\cite{routerbench, mixLLM2025, zooter, frugalgpt2024}, demonstrating consistent performance-cost trade-offs. Despite this promise, current inference-time model selection methods still share two limitations. 

First, they \textit{do not directly enforce a workload-level budget} constraint. Majority of prior works either enforce no hard budget (e.g., \cite{hybridllm,routellm,zooter}) or constrain cost only per-query (e.g., \cite{frugalgpt2024, UnifiedApproachRouting2025, metallm}). 
However, real-world deployments operate under \textit{workload-level} resource constraints: API token quotas or monthly subscription fees apply across an entire set of queries, not individually. 
LLM selection algorithms that directly reason at the workload-level can prioritize complex or high utility queries for larger and expensive models while being conservative on easier queries, something that per-query constrained algorithms cannot do by construction.

Second, most algorithms need \textit{expensive supervision signals.} Existing inference-time LLM selection methods typically follow a train-deploy paradigm which requires dense supervision: outputs and costs must be collected for every query-model pair before training a routing policy~\cite{hybridllm, routellm}, which becomes prohibitively expensive as the number of model candidates or training queries increases. 
We focus on \textbf{LLM routing} because routing each query to exactly one model avoids the redundant inference costs of cascading, enabling tighter budget constraint. LLM Routing maps naturally to the multi-armed bandit setting (Section~\ref{subsec:formulation}). These limitations lead to two central questions:
\textit{(1) How can we perform LLM routing w.r.t. a workload-level budget constraint? (2) How can routing be learned when query-model interactions are sparse or are gradually being observed?}  
\begin{figure}[ht]
    \centering
    \includegraphics[width=\columnwidth]{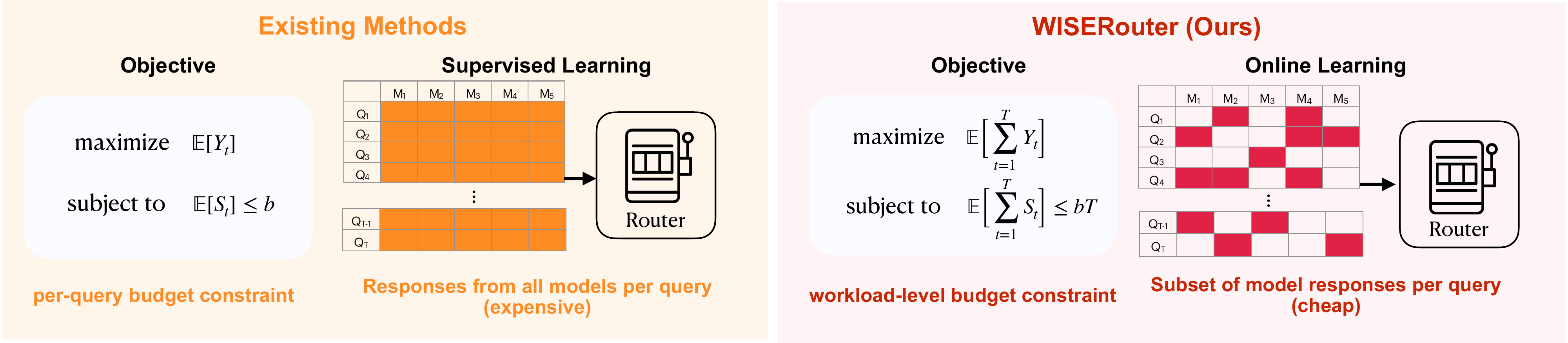}
    \caption{Comparison of \optiroute with previous LLM selection methods.}
    \label{fig:routing_comparison}
\end{figure}

We introduce \textbf{\optiroute} (\textbf{W}orkload-\textbf{I}nformed \textbf{S}equential \textbf{E}fficient \textbf{Router}, WR for short), a unified framework that addresses both limitations by formulating LLM routing as a Constrained Contextual Bandit (CCB)~\cite{resourcefulbandit} problem, where each query embedding is the context, the selected LLM is the action, and the observed response quality is the reward. The goal is to maximize cumulative expected reward while satisfying a workload-level expected cost budget over the horizon. \optiroute addresses the first limitation of enforcing a workload-level budget via Adaptive Linear Programming (ALP)~\cite{ucb-alp}, which translates the remaining budget into a per-round constraint at each query step, enabling dynamic budget allocation across queries by resolving a lightweight linear program without retraining. To address the second limitation and learn from sparse supervision signal, we propose \efalp, which jointly learns reward and cost statistics online via an $\epsilon$-first exploration before the ALP-based exploitation, all under a single shared budget. We prove the end-to-end regret bound for \efalp, covering both phases under a workload level budget constraint. Figure~\ref{fig:routing_comparison} illustrates how \optiroute differs from prior methods along these two dimensions.

Concretely, \optiroute operates in two stages: query embedding discretization followed by ALP-based model selection, with \alp for data-rich settings (e.g., responses from all candidate models collected per query) and \efalp for data-scarce settings (e.g., responses from only a subset of candidate models per query). On datasets RouterBench~\cite{routerbench} and SWE-Bench~\cite{swebench}, \alp consistently matches or outperforms the baselines. 
\alp improves average quality by $14\%$ over the best baseline on SWE-Bench at the tightest budget.
\alp terminates with the lowest unspent budget across most budget levels without overspending, training over $5,900\times$ faster than neural-network based baselines for each new budget setting. In the data-scarce setting, \efalp matches offline methods while reducing the required training data by $90\%$ on RouterBench.

Our contributions are: (i) To the best of our knowledge, we are the first to formulate workload-constrained LLM routing as a CCB problem and develop \alp, which enforces a global budget via solving ALP and adapts to new budget constraints without retraining. (ii) We propose \efalp, an online variant that jointly learns reward and cost statistics from sparse interactions, and prove the first end-to-end regret bound for bandit algorithm applied on LLM routing covering both exploration and exploitation under a shared workload budget. (iii) We empirically validate \optiroute on RouterBench and SWE-Bench, demonstrating superior performance at tight budgets, precise budget utilization, and significant reduction in training data with \efalp.

\section{Related Work}
\label{sec:related_work}

Efficient LLM inference is a central challenge in large-scale AI deployment, with recent work spanning algorithmic improvements such as speculative decoding~\cite{speculativedecode, speculativedecode2}, pruning~\cite{modelpruning1, modelpruning2}, quantization~\cite{quantization1, quantization2}, efficient KV caching~\cite{pyramidkv}, and others~\cite{s3, moe1, moe2}. Orthogonally to these techniques which have access to model internal architectures, inference-time model selection reduces inference cost by routing queries to black-box models of varying capability and price~\cite{hybridllm,routellm,zooter}, cascading through a fixed model sequence~\cite{frugalgpt2024,UnifiedApproachRouting2025}, or ensembling multiple outputs~\cite{llmblender,thriftllm}. Below, we focus on the two directions most directly related to \optiroute; a full survey appears in Appendix~\ref{app:related_work}.

\noindent\textbf{Efficient LLM Selection with Cost Constraint.}
Prior work on cost-aware LLM selection has addressed budget constraints in different ways, but none jointly enforces a workload-level budget across both learning and deployment. MetaLLM~\cite{metallm} treats model selection as a contextual bandit, using UCB~\cite{ucb} to satisfy per-query cost constraint by tuning a cost scaling factor in reward function. FrugalGPT~\cite{frugalgpt2024} and CascadeRouting~\cite{UnifiedApproachRouting2025} cascade queries through increasingly capable models until a quality threshold is met, enforcing an explicit per-query budget. Such per-query constraints can be myopic, and adapting these methods to different cost constraints typically requires retraining the routing policy. PILOT~\cite{PILOT2025} combines a LinUCB router with an online knapsack solver for budget enforcement, but applies the budget only during deployment: specifically, the exploration phase with LinUCB is unconstrained, and the regret analysis covers router learning alone rather than the full stages. \efalp, the online variant of \optiroute, is the first LLM routing method to enforce a single workload-level budget across both exploration and exploitation, and to provide an end-to-end regret bound covering the complete online routing process.

\noindent\textbf{Constrained Contextual Bandits.}
\citet{resourcefulbandit} establishes the theoretical foundation for budget-constrained contextual bandits with an $\tilde{O}(\sqrt{T})$ regret bound, but the proposed algorithm is computationally intractable in practice. \citet{ucb-alp} proposes UCB-ALP for uniform-cost settings and $\epsilon$-first ALP for heterogeneous costs setting, both assuming cost statistics are known a priori. Our work extends the $\epsilon$-first ALP algorithm to the regime where both reward and cost must be estimated online from sparse interactions under a single budget shared across exploration and exploitation.
\section{Preliminaries}

\subsection{LLM routing as Constrained Contextual Bandit}
\label{subsec:formulation}

We consider a routing setting with $K$ black-box LLMs, indexed by $\mathcal{K}=\{1,\ldots,K\}$, and a stream of user queries $\mathcal{Q}$. At each time step $t$, an input prompt $q_t \in \mathcal{Q}$ arrives and the router selects a single model $A_t\in\mathcal{K}$ to generate the response.
After invoking the selected model, the router can observe:
(i) a performance score $Y_t\in[0,1]$ that quantifies response quality (e.g., human preference, ground-truth-based metric when available), and
(ii) a cost score $S_t \in [0,1]$ that measures the resources used by the selected model, normalized by a known per-query maximum. In this work, $S_t$ is the monetary cost computed by token-based pricing of the selected LLM. The same formulation also applies to other cost notions, such as latency or energy consumption. 

We formulate LLM routing as a Constrained Contextual Bandit (CCB) problem. Each incoming query provides a context, represented by its query embedding (detailed in Section~\ref{subsec: clustering}), and each LLM selection is an action. The router receives bandit feedback 
$(Y_t,S_t)$ only for the selected model at timestep $t$.
Our goal for a routing policy $\Psi$ is to maximize expected cumulative performance subject to a workload-level budget:
\[
\text{maximize}\quad \;
       U_{\Psi}(T,B) = \mathbb{E}\Bigl[\sum_{t=1}^T Y_t\Bigr]
\quad \text{subject to} \quad
\mathbb{E}\Bigl[\sum_{t=1}^{T} S_t\Bigr]\le B.
\]
We optimize expected reward subject to expected cost to account for stochasticity in model generation, following prior work on efficient LLM selection~\cite{frugalgpt2024,UnifiedApproachRouting2025}.

\subsection{Adaptive Linear Programming}
\label{subsec:alp-prelim}
\citet{ucb-alp} propose ALP for constrained contextual bandits over a finite horizon $T$: at each step it solves a linear program that converts the remaining budget $b$ across the remaining $\tau$ rounds into a per-round average constraint $b/\tau$, adapting to cost consumption at runtime. It operates under two assumptions: (a) a \emph{finite context space} $\mathcal{J} = \{1,\ldots,J\}$ with known context distribution 
% $\pi_j = \mathbb{P}(X_t = j)$
$\{\pi_j\}_{j=1}^J$
; and (b) \emph{known reward and cost statistics} $u_{j,k}$ and $c_{j,k}$, denoting the expected reward and cost for each context-action pair $(j,k)$.
Let $p_{j,k} \in [0,1]$ denote the probability of selecting action $k$ given context $j$.
A null action $k=0$ with $u_{j,0}=c_{j,0}=0$ (skipping the current context) guarantees a valid action always exists.   
At each time step, ALP solves: 
\begin{align}
(\mathrm{LP}_{\tau,b})\quad
\max_{p_{j,k}} \quad
& \sum_{j=1}^{J}\pi_j \sum_{k=1}^{K} p_{j,k}\,u_{j,k} \label{eq:lp-obj} \\
\text{s.t.}\quad
& \sum_{j=1}^{J}\pi_j \sum_{k=1}^{K} p_{j,k}\,c_{j,k} \;\le\; \frac{b}{\tau},
    \label{eq:lp-budget} \\ 
    & \sum_{k=1}^{K} p_{j,k} \;\le\; 1,\quad \forall\,j\in\mathcal{J}, \label{eq:lp-prob} \\ 
    & p_{j,k} \in [0,1]. 
  \end{align} 
 ALP was designed for finite contexts with known statistics. Neither assumption holds directly in LLM routing: query embeddings are continuous and high-dimensional, and reward and cost statistics are not known a priori and must be estimated from data. Section~\ref{sec:method} addresses both limitations. 

\section{Our Method}
\label{sec:method}

\optiroute (Figure~\ref{fig:workflow}) addresses the two limitations identified in Section~\ref{subsec:alp-prelim} in two steps: (1) discretize the query embedding space into a finite context set compatible with ALP (Section~\ref{subsec: clustering}); and (2) perform ALP-based model selection under two regimes - \alp (Section~\ref{subsec:optiroute-alp}), which estimates statistics from historical data, and \efalp (Section~\ref{subsec:efalp}), which learns them online under a shared workload budget. We close with a discussion of the assumptions of our formulation. 

\begin{figure}[ht]
    \centering
    \includegraphics[width=0.9\columnwidth]{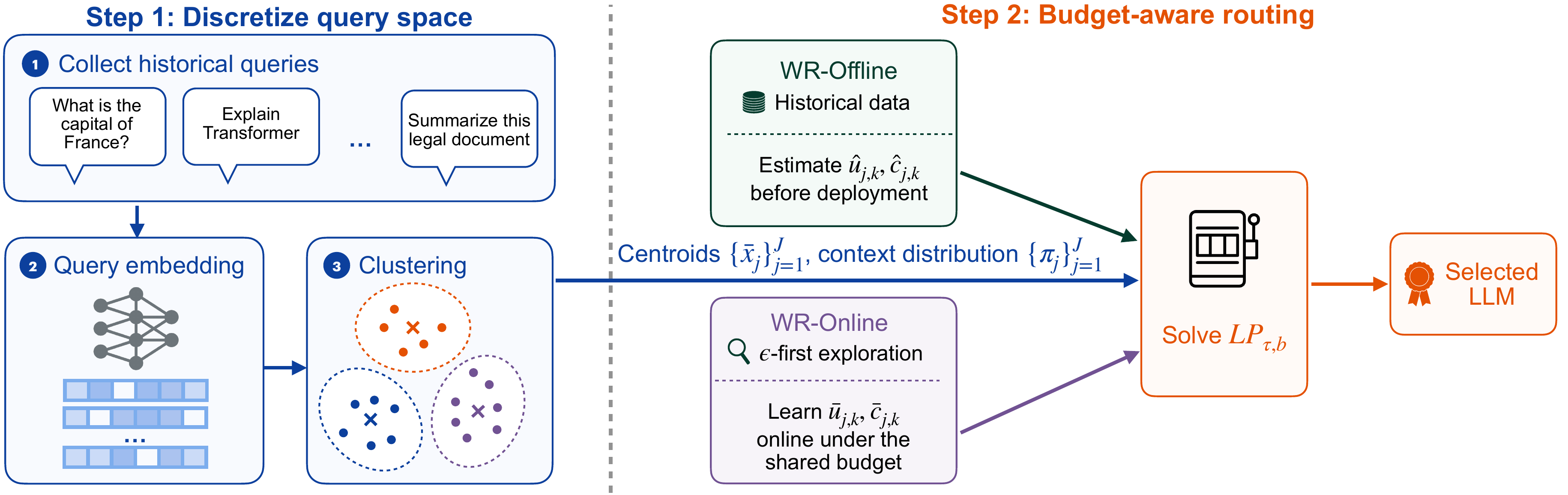}
    \caption{Overview of \optiroute workflow.}
    \label{fig:workflow}
\end{figure}

\subsection{Context Discretization}
\label{subsec: clustering}
For the CCB formulation to be meaningful, the context must encode query-level features that explain the variation in LLM performance across queries. Since different LLMs can perform unevenly across query semantics~\cite{mixLLM2025}, we use a pretrained encoder $g(\cdot)$ to map each query $q\in\mathcal{Q}$ to an embedding vector $x=g(q)\in\mathbb{R}^d$. However, ALP assumes a finite context space with an explicit context distribution, so we discretize the embedding space via clustering. Specifically, given a historical query set $\mathcal{Q}=\{q_i\}_{i=1}^n$, we compute embeddings $x_i=g(q_i)$ and cluster $\{x_i\}_{i=1}^n$ into $J$ groups with centroids $\{\bar{x}_j\}_{j=1}^J$. 
We define the discrete context of query $q_i$ as its cluster assignment $\sigma(g(q_i))\in\{1,\dots,J\}$, and estimate the context distribution empirically by $\pi_j=\mathbb{P}(X=j)=|\{i\in[n]: \sigma(g(q_i))=j\}|/n.$
At test time, each incoming query $q_t$ is embedded as $x_t=g(q_t)$ and assigned to its nearest centroid, yielding a discrete context $X_t\in\{1,\dots,J\}$. We assume that for each LLM, queries in the same cluster share the same expected reward and expected cost. The granularity of this approximation is controlled by $J$. Ablations in Appendix~\ref{app:ablation} show stable performance across cluster counts and algorithms, with small intra-cluster reward and cost variance (Appendix~\ref{app:cluster_var}).

\subsection{\alp with Known Statistics}
\label{subsec:optiroute-alp}
\alp addresses the data-rich regime, where a historical dataset $\mathcal{D}_{\mathrm{hist}}$ of query-model interactions is available. Applying the discretization from Section~\ref{subsec: clustering} to $\mathcal{D}_{\mathrm{hist}}$, we estimate $\hat u_{j,k}$ and $\hat c_{j,k}$ as the mean reward and cost for each context-action pair $(j,k)$. These statistics are computed once offline and fixed at deployment. At each routing step, the router solves $\mathrm{LP}_{\tau,b}$ via a standard LP solver with negligible additional latency and samples an action from solution probability $p_{j,\cdot}(b/\tau)$ across all models under context $j$ at per-round budget $b/\tau$. Since the budget constraint is on expected cost, remaining budget $b$ is updated by mean cost. The full procedure is given in Algorithm~\ref{alg:offline-alp}. 

    \begin{algorithm}[H] % [H] forces it to stay HERE
\caption{\textsc{\alp}}
\label{alg:offline-alp}
      \begin{algorithmic}[1]
\REQUIRE Horizon $T$, workload budget $B$, model set $\mathcal{K}$, encoder $g$, clustering map $\sigma:\mathbb{R}^d\!\to\!\{1,\dots,J\}$, context distribution $\{\pi_j\}_{j=1}^J$, historical data $\mathcal{D}_{\mathrm{hist}}$

\STATE \textbf{Init:} $\tau\gets T$,  $b\gets B$; for each $(j,k)$, compute $\hat u_{j,k}$ and $\hat c_{j,k}$ from $\mathcal{D}_{\mathrm{hist}}$

\FOR{$t=1$ \TO $T$ and $b>0$}
        \STATE Observe query $q_t$; assign context $j\gets \sigma(g(q_t))$
        \STATE Solve $\mathrm{LP}_{\tau,b}$ using $(\hat u_{j,k},\hat c_{j,k})$ and $\{\pi_j\}_{j=1}^J$; Obtain solution probabilities $p_{j,k}(b/\tau)$ at budget $b/\tau$ for each $(j,k)$; Sample $A_t \sim p_{j,\cdot}(b/\tau)$ and route $q_t$ to model $A_t$
        \STATE $b \gets b - \hat c_{j,A_t}$; $\tau \gets \tau-1$
    % \ENDIF
\ENDFOR
\end{algorithmic}
\end{algorithm}

\subsection{\efalp with Exploration}
\label{subsec:efalp}

When model responses have not been pre-collected across all candidate models for the historical queries, either because they are expensive or are not available beforehand, supervised learning is infeasible. \efalp instead learns mean reward and cost online all under the same shared budget $B$. Both statistics must be estimated post-hoc with the generated response from each model. 

\efalp extends the $\epsilon$-first ALP in \citet{ucb-alp} to jointly estimate $u_{j,k}$ and $c_{j,k}$ online. During the exploration phase of $\epsilon(T)$ steps, it routes each query to the least-invoked model within its context, ensuring every $(j,k)$ pair is visited at least $\lfloor\epsilon(T)/K\rfloor$ times, and updates running mean estimates $\bar{u}_{j,k}$, $\bar{c}_{j,k}$ from observed feedback. For the remaining $(1-\epsilon)T$ exploitation steps, the router solves $\mathrm{LP}_{\tau,b}$ with the estimated statistics. Section~\ref{sec:regret} derives the sufficient condition on $\epsilon(T)$ for an $O(\sqrt{T})$ end-to-end regret bound. The full procedure is in Algorithm~\ref{alg:eps-first-alp}.

\begin{algorithm}[h!]
\caption{\textsc{\efalp}}
\label{alg:eps-first-alp}
\begin{algorithmic}[1]

\REQUIRE Horizon $T$, budget $B$, exploration length $\epsilon(T)$, model set $\mathcal{K}$, encoder $g$, clustering map $\sigma:\mathbb{R}^d\!\to\!\{1,\dots,J\}$, context distribution $\{\pi_j\}_{j=1}^J$

\STATE Initialize $N_{j,k},\bar u_{j,k},\bar c_{j,k}\gets 0$ for all $j,k$; $b\gets B$, $\tau\gets T$,

\FOR{$t=1,\ldots,\varepsilon(T)$ \textbf{and} $b>0$}
        % \STATE Observe an incoming query $X_t$, encode it to embedding $q_t$ and assign it to context $j$
        \STATE Observe a query $q_t$; set $j\gets \sigma(g(q_t))$
        \STATE $A_t = \arg\min_{k\in\mathcal K} N_{j,k}$ with random tie-break; observe $(Y_{A_t,t}, S_{A_t,t})$
        \STATE $b\gets b - S_{A_t,t}$; $\tau \gets \tau-1$; $N_{j,k}\!\gets\!N_{j,k}+1$, $\bar u_{j,k}\!\gets\!\bar u_{j,k}+\frac{Y_{A_t,t}-\bar u_{j,k}}{N_{j,k}}$, $\bar c_{j,k}\!\gets\!\bar c_{j,k}+\frac{S_{A_t,t}-\bar c_{j,k}}{N_{j,k}}$
    % \ENDIF
\ENDFOR
\FOR{$t = \varepsilon(T)+1$ \textbf{to} $T$ \textbf{and} $b>0$}
        \STATE Observe a query $q_t$, set $j\gets \sigma(g(q_t))$ 
        % and $\tau \gets T - t + 1$
        \STATE Solve $\text{LP}_{\tau,b}$ with $(\bar u_{j,k}, \bar c_{j,k})$ to obtain $p_{j,\cdot}(b/\tau)$; 
        \STATE Sample $A_t \sim p_{j,\cdot}(b/\tau)$ and route $q_t$ to model $A_t$
        \STATE $b \gets b - \bar{c}_{j,A_t}$; $\tau \gets \tau-1$

\ENDFOR
\end{algorithmic}
\end{algorithm}

\noindent\textbf{Practical validity of assumptions.} 
The CCB formulation (Section~\ref{subsec:formulation}) rests on three assumptions that hold under standard deployment conditions. \emph{I.i.d.\ contexts:} queries from a stable user population are well-approximated by the stationary distribution ${\pi_j}$ estimated via clustering. \emph{Finite horizon $T$:} corresponds to a planning window (e.g., a monthly API quota) estimable from historical traffic. \emph{Feedback acquisition:} token-count costs are available immediately after generation; quality scores are read from $\mathcal{D}_{\mathrm{hist}}$ in \alp and collected during exploration without blocking routing decisions in \efalp. 
Appendix~\ref{app:assumptions} discusses practical mitigation for distributional shift and reward feedback.

\section{Regret Analysis for \efalp}
\label{sec:regret}
Regret measures the cumulative expected reward gap between a bandit policy and an oracle policy with full knowledge of all statistics~\cite{ucb-alp,ucb}. We analyze the regret of \efalp by first establishing an oracle upper bound, then deriving a sufficient condition on the exploration length, and finally proving an end-to-end bound that jointly covers exploration and exploitation under a shared budget. 
Since this condition depends on unknown ground-truth statistics, we further provide a confidence level test as a criterion for certifying when exploration is theoretically sufficient.

% \subsection{Regret of \alp}
% \label{subsec:regret-alp}
\noindent\textbf{Oracle upper bound.} 
Let $U^*(T,B)$ denote the cumulative expected reward of the oracle policy with knowledge of the true statistics ${u_{j,k}, c_{j,k}}$ over horizon $T$ under budget $B$. We establish an upper bound via $\mathrm{LP}_{T,B}$ which uses the same formulation as $\mathrm{LP}_{\tau,b}$ (Section~\ref{subsec:optiroute-alp}) but with a fixed per-round budget $\rho = B/T$ in place of the adaptive~$b/\tau$. 
We prove that $\widehat{U}(T,B)$ is an upper bound of $U^*(T,B)$ in Appendix~\ref{app:upperbound}. The regret of \alp is thus bounded by $R_{\text{ALP}}(T,B) \leq \widehat{U}(T,B) - U_{\text{ALP}}(T,B)$.

\noindent\textbf{Restating regret of ALP.}
The solution of upper bound algorithm $\mathrm{LP}_{T,B}$ has a greedy structure: given per-round budget constraint $\rho$, it assigns probability 1 to context-action pairs in non-increasing order using their reward-cost ratio (defined below) until the remaining budget $\rho$ can no longer support probability 1 for the next pair's expected cost, which then receives a fractional probability so that its expected cost exactly exhausts the budget. A budget $\rho$ is \emph{boundary} if it coincides exactly with a switching point in this ordering; otherwise it is \emph{non-boundary}. ALP's regret arises from fluctuations in $b/\tau$ around $\rho$: for non-boundary $\rho$, small fluctuations leave the solution probability unchanged w.h.p., incurring $O(1)$ regret; for boundary $\rho$, fluctuations may shift the probability allocation, yielding $O(\sqrt{T})$ regret. 
We restate the ALP regret bound from~\cite{ucb-alp} in Theorem~\ref{thm:alp_regret_simplified}. The full theorem statement and proof sketch are provided in Appendix~\ref{app:detial_alp_regret}. 

\begin{theorem}
\label{thm:alp_regret_simplified}
Given any fixed $\rho=B/T$, the regret of ALP with known statistics satisfies:
(1) for non-boundary cases, $R_{\text{ALP}}(T,B)=O(1);$ and (2) for boundary cases, $R_{\text{ALP}}(T,B)=O(\sqrt{T}).$
\end{theorem}

% \subsection{Regret Analysis for \efalp}
% \label{subsec:regret-efalp}

\noindent\textbf{Sufficient exploration for \efalp.}
In \efalp, $u_{j,k}$ and $c_{j,k}$ are unknown and estimated online. Meanwhile, the ALP solution relies on an ordered list of context-action pairs determined using the pairwise ratios $\xi_{j,k_1,k_2} := (u_{j,k_1} - u_{j,k_2})/(c_{j,k_1} - c_{j,k_2})$. If empirical estimates $\bar{u}_{j,k}$, $\bar{c}_{j,k}$ induce a different ordering of $\bar\xi_{j,k_1,k_2} := (\bar{u}_{j,k_1} - \bar{u}_{j,k_2})/(\bar{c}_{j,k_1} - \bar{c}_{j,k_2})$ when exploration ends, the solution during exploitation will be suboptimal compared with the ground truth solution of ALP. If exploration instead produces estimates accurate enough to recover the correct ordering, exploitation runs ALP with the right solution structure and inherits the regret bound of Theorem~\ref{thm:alp_regret_simplified}. 

Our goal is therefore to characterize how much exploration suffices to guarantee correct ordering with high probability. The following quantities characterize how difficult this ordering is to recover from finite samples. We assume $c_{j,k_1}\neq c_{j,k_2}$ for all $j$ and actions $k_1\neq k_2$, and define three quantities: the minimal cost gap 
% $\Delta_{\min}^{c} :=\min_{j\in\mathcal J}\min_{k_1\neq k_2\in\{0\}\cup\mathcal K}\bigl|c_{j,k_1}-c_{j,k_2}\bigr|$, 
$\Delta_{\min}^{c} :=\min_{j\in\mathcal J, k_1\neq k_2\in\{0\}\cup\mathcal K}\bigl|c_{j,k_1}-c_{j,k_2}\bigr|$,
the minimal ratio separation $\Delta_{\min}^{\xi} = \min_{j_1,j_2 \in \mathcal{J},\, k_{11},k_{12},k_{21},k_{22} \in \{0\}\cup\mathcal{K}} |\xi_{j_1,k_{11},k_{12}} - \xi_{j_2,k_{21},k_{22}}|$, and $\pi_{\min}:=\min_{j\in \mathcal{J}}\pi_j$. Small values require longer exploration to achieve correct ordering because the tighter cost gaps and ratio separations require more samples to distinguish reliably, while rarer contexts accumulate observations more slowly. However, smaller separations also imply that the context-action pairs have similar \(\xi\), so misorderings among them are less consequential.
Lemma~\ref{lem:sufficient_explore} shows that if $\epsilon(T)$ is sufficient, then $\epsilon$-first ALP recovers the correct ordering of $\xi_{j,k_1,k_2}$ with high probability after exploration.

\begin{lemma}[Sufficient exploration]
  \label{lem:sufficient_explore}       
  Let $0 < \delta < 1$ and $L := (\Delta_{\min}^c)^2\Delta_{\min}^\xi$. Under \efalp, if 
  \begin{equation}                       \epsilon(T) \;\geq\; \Bigl\lceil \frac{K}{(1-\delta)\pi_{\min}} +       \log T \cdot \max \;\Bigl\{ \frac{1}{\delta^2},\;                 
  \frac{128K}{(1-\delta)\pi_{\min}L^2} \Bigr\} \Bigr\rceil,                   \label{eq:eps-condition}              \end{equation}                        then for any contexts $j_1, j_2 \in \mathcal{J}$ and actions $k_{11}, k_{12}, k_{21}, k_{22} \in \{0\}\cup\mathcal{K}$, if $\xi_{j_1,k_{11},k_{12}} < \xi_{j_2,k_{21},k_{22}}$, then at the end of the $\epsilon(T)$-th round, $\mathbb{P}\!\left[\,\bar{\xi}_{j_1,k_{11},k_{12}} \geq
  \bar{\xi}_{j_2,k_{21},k_{22}}\right] \leq (16+J)T^{-2}.$  
  Moreover, the algorithm ranks all the $\xi_{j,k_1,k_2}$'s correctly with probability no less than $1-(4KJ+J)T^{-2}$.         
  \end{lemma}

Absent in prior regret analyses with known costs, the core challenge is that both $u_{j,k}$ and $c_{j,k}$ must be estimated simultaneously, and errors in both propagate through the ratio $\xi_{j, k_1, k_2}$. Our proof (Appendix~\ref{app:sufficient_explore}) reformulates ratio comparison as an equivalent sign test on the difference of their cross-products, avoiding random denominators entirely. We apply Hoeffding's inequality to the individual empirical means $\bar{u}_{j,k}$ and $\bar{c}_{j,k}$, and the resulting concentration bounds propagate to the cross-product difference via triangle inequality.

\begin{theorem}[End-to-end regret of \efalp]
\label{thm:efalp-regret}
Let $0 < \delta < 1$. Under \efalp, if $\epsilon(T)$ satisfies condition \eqref{eq:eps-condition}, then the regret of \efalp satisfies:
\textup{(1)} for non-boundary cases, $R_{\textup{\efalp}}(T,B) = O(\log T)$;\; \textup{(2)} for boundary cases, $R_{\textup{\efalp}}(T,B) = O(\sqrt{T})$.
\end{theorem}

\noindent\textbf{Proof sketch.}
When rankings are correct, the exploitation phase runs ALP with correctly ordered statistics and incurs the regret from Theorem~\ref{thm:alp_regret_simplified}. The exploration phase contributes an additional $O(\log T)$ regret. Combined, the total is $O(\log T)$ in the non-boundary case and $O(\sqrt{T})$ in the boundary case since $O(\sqrt{T})$ dominates $O(\log T)$. When rankings are incorrect, the probability of incorrect ordering is $O(T^{-2})$ by Lemma~\ref{lem:sufficient_explore} that contributes a total of $O(T)\cdot O(T^{-2})=O(T^{-1})$ to regret, which is negligible. \qed

% \subsection{Confidence Level Test for \efalp}
% \label{sec:CLT}
\noindent\textbf{Confidence level test.} 
We further derive a confidence level test (Appendix~\ref{app:clt}) that certifies when the estimated statistics are accurate enough to guarantee correct ordering of $\xi_{j,k_1,k_2}$ w.h.p., providing a theoretically grounded stopping criterion. As the test is conservative in practice because reward and cost gaps in LLM routing are often small, our experiments use a fixed exploration schedule instead. 

\section{Experiments}
\label{sec:experiment}

We evaluate \optiroute under varying workload-level budgets. Section~\ref{sec: offline_res} compares \alp against cost-constrained LLM selection baselines on performance, budget adherence, and latency. Section~\ref{sec:online_res} evaluates \efalp's learning efficiency against offline methods and includes simulation results on empirical regret and budget consumption. Ablation study is provided in Appendix~\ref{app:ablation}.

\subsection{LLM Routing Experiment Setup}
\label{subsec:exp_setup}

\noindent\textbf{Datasets.} 
We evaluate the LLM routing baselines on RouterBench~\cite{routerbench} and SWE-Bench-verified~\cite{swebench} (SWE-Bench). RouterBench contains $36,497$ queries from eight diverse datasets, while SWE-Bench includes 500 realistic questions from real GitHub issues. All datasets use a 50:50 train-test split. Further details, including single-model performance, are provided in Appendix~\ref{app:dataset}.

\noindent\textbf{Implementation Details for \optiroute.}
We embed queries using OpenAI's \textsc{text-embedding-3-small} ($d=1536$) and discretize via K-means~\cite{kmeans1967} with $J=16$, selected on validation sets; Appendix~\ref{app:ablation} shows robustness to these choices. Deployment budgets range from the cost of routing all queries to the cheapest model to that of always routing all to the most expensive. For \alp, $T$ is the number of test queries with statistics estimated from the training set. For \efalp, $T$ spans both training and test queries; exploration is restricted to the training set for fair comparison with offline methods, with exploration length set to $1\times$ -- $4\times$ the number of queries in training set. Unlike offline training which requires all model responses per query, each exploration step invokes only one model, keeping the total exploration cost well below the offline baseline.

\noindent\textbf{Evaluation Metrics.}
For each dataset, we use the per-query performance scores and monetary costs as the reward and cost observations; details are in Appendix~\ref{app:dataset}.

\noindent\textbf{Baselines.}
We compare \optiroute against the following baselines, all of which are inference-time LLM selection methods under cost constraints: (1) \textbf{FrugalGPT}~\cite{frugalgpt2024} -- an LLM cascading method that constrains on a per-query budget; (2) \textbf{CascadeRouting}~\cite{UnifiedApproachRouting2025} -- a unified LLM cascade and routing method, with a per-query budget constraint; (3) \textbf{MetaLLM}~\cite{metallm} -- a bandit-based routing method that controls the per-query cost as a cost-scaling penalty in the reward function and uses UCB~\cite{ucb} for model selection; (4) \textbf{SingleBest} -- for each context $j$, greedily selects the highest-reward action within a per-query budget.
% ;  (5) \textbf{StaticLP} -- solves $\mathrm{LP}_{T,B}$ once using the fixed per-query budget $B/T$, then samples a concrete action at each round given the incoming context. 
All baseline implementation details are provided in Appendix~\ref{app:baseline}. \footnote{PILOT~\cite{PILOT2025} is not publicly available; we therefore omit it from experiments and discuss it analytically in Section~\ref{sec:related_work}.}

\subsection{Offline Baseline Comparison}

\begin{table*}[t]
\centering
\caption{Performance comparison under varying budget constraints (\$). SingleBest, MetaLLM, and \alp involve randomness and are averaged over five runs;  FrugalGPT and CascadeRouting are deterministic and reported as single runs. Best results in \textbf{bold}.}
\label{tab:offline_performance_comparison}
\footnotesize
\setlength{\tabcolsep}{4pt}
\resizebox{\textwidth}{!}{
\begin{tabular}{l
cccccc
cccccc}
\toprule
& \multicolumn{6}{c}{\textbf{RouterBench}} 
& \multicolumn{6}{c}{\textbf{SWE-Bench}} \\
\cmidrule(lr){2-7} \cmidrule(lr){8-13}

\textbf{Method}
& \multicolumn{1}{c}{\scriptsize B=2.52}
& \multicolumn{1}{c}{\scriptsize B=19.92}
& \multicolumn{1}{c}{\scriptsize B=37.31}
& \multicolumn{1}{c}{\scriptsize B=54.71}
& \multicolumn{1}{c}{\scriptsize B=72.10}
& \multicolumn{1}{c}{\scriptsize B=151.42}
& \multicolumn{1}{c}{\scriptsize B=232.72}
& \multicolumn{1}{c}{\scriptsize B=253.04}
& \multicolumn{1}{c}{\scriptsize B=273.37}
& \multicolumn{1}{c}{\scriptsize B=293.69}
& \multicolumn{1}{c}{\scriptsize B=314.02}
& \multicolumn{1}{c}{\scriptsize B=659.44} \\

\midrule
FrugalGPT      
& 0.00 & 0.70 & 0.70 & 0.72 & 0.73 & 0.80 
& 0.00 & 0.00 & 0.00 & 0.42 & \textbf{0.67} & \textbf{0.68} \\

CascadeRouting 
& 0.50& 0.71 & 0.73 & 0.75 & 0.76 & 0.79 
& 0.51 & 0.55 & 0.55 & 0.56 & 0.56 & 0.67 \\

SingleBest
& 0.36$_{\pm 0.011}$
& 0.69$_{\pm 0.002}$
& 0.70$_{\pm 0.002}$
& 0.70$_{\pm 0.001}$
& 0.71$_{\pm 0.002}$
& 0.79$_{\pm 0.001}$ 
& 0.24$_{\pm 0.037}$
& 0.31$_{\pm 0.064}$
& 0.36$_{\pm 0.095}$
& 0.40$_{\pm 0.078}$
& 0.46$_{\pm 0.054}$
& 0.65$_{\pm 0.020}$ \\

MetaLLM
& \textbf{0.55}$_{\pm 0.003}$
& 0.70$_{\pm 0.004}$
& 0.72$_{\pm 0.020}$
& 0.74$_{\pm 0.003}$
& 0.74$_{\pm 0.003}$
& \textbf{0.81}$_{\pm 0.001}$
& 0.51$_{\pm 0.035}$
& 0.51$_{\pm 0.035}$
& 0.52$_{\pm 0.103}$
& 0.61$_{\pm 0.057}$
& 0.65$_{\pm 0.018}$
& 0.65$_{\pm 0.018}$ \\

\rowcolor{methodshade}
\alp           
& 0.54$_{\pm 0.003}$
& \textbf{0.72}$_{\pm 0.002}$
& \textbf{0.74}$_{\pm 0.003}$
& \textbf{0.76}$_{\pm 0.002}$
& \textbf{0.77}$_{\pm 0.001}$
& \textbf{0.81}$_{\pm 0.001}$
& \textbf{0.58}$_{\pm 0.033}$
& \textbf{0.60}$_{\pm 0.027}$
& \textbf{0.62}$_{\pm 0.031}$
& \textbf{0.64}$_{\pm 0.020}$
& 0.65$_{\pm 0.020}$
& 0.65$_{\pm 0.017}$ \\

\bottomrule
\end{tabular}
}
\end{table*}

\begin{figure}[t]
\centering
\includegraphics[width=0.65\columnwidth]{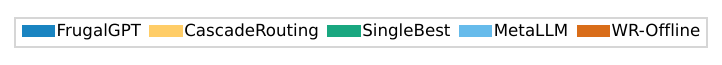}
\begin{minipage}[b]{0.49\columnwidth}
    \centering
    \includegraphics[width=\linewidth]{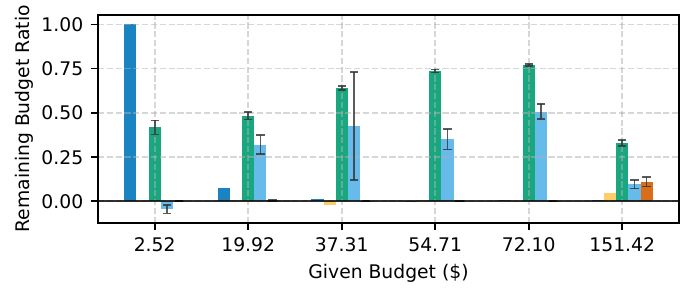}
    
    \small (a) Budget Utilization Analysis for RouterBench
\end{minipage}\hfill
\begin{minipage}[b]{0.49\columnwidth}
    \centering
    \includegraphics[width=\linewidth]{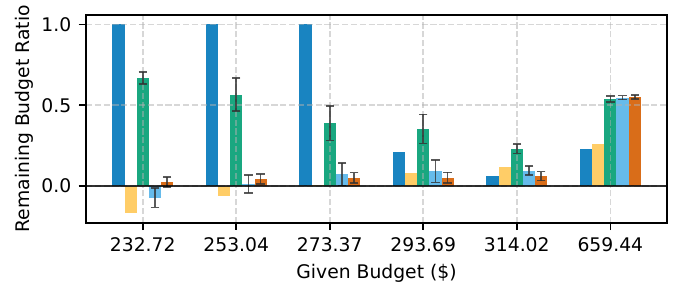}
    
    \small (b) Budget Utilization Analysis for SWE-Bench
\end{minipage}

\caption{Remaining budget ratios at each budget level. Error bars show $\pm1$ std over five runs.}
\label{fig:combined_remain_budget}
\end{figure}

\begin{figure}[t]
    \centering
    % --- Left Side: Table ---
    \begin{minipage}{0.48\linewidth}
        \centering
        \captionof{table}{Latency comparison of LLM routing baselines on RouterBench dataset averaged over all budget settings.}
        \label{tab:latency_routerbench}
        \footnotesize
        \begin{tabular}{l c c}
            \noalign{\hrule height 1pt}
            \multirow{2}{*}{Method} & Train & Inference \\
            & Latency (s) & Latency/Query (s) \\
            \hline
            FrugalGPT       & $8.85 \times 10^{3}$ & $1.51 \times 10^{-2}$ \\
            CascadeRouting  & $4.19 \times 10^3$   & $2.07\times 10^{-3}$ \\
            SingleBest      & 1.50                 & $2.25 \times 10^{-4}$ \\
            MetaLLM         & 1.76                 & $5.35 \times 10^{-3}$ \\
            \alp            & 1.50                 & $1.56 \times 10^{-3}$ \\
            \noalign{\hrule height 1pt}
        \end{tabular}
    \end{minipage}
    \hfill % Adds spacing between the two minipages
    % --- Right Side: Figure ---
    \begin{minipage}{0.49\linewidth}
        \centering
        \includegraphics[width=\linewidth]{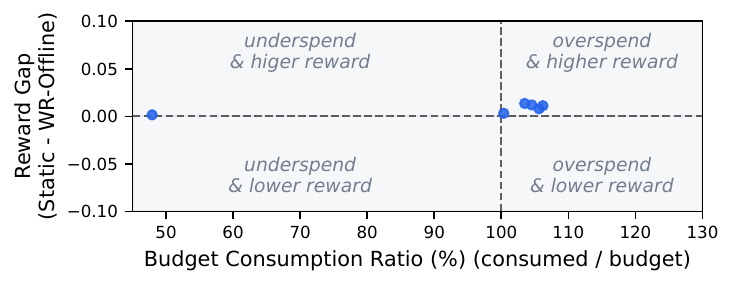} % Replace with your image filename
      
        \caption{Budget consumption vs.\ reward gap for the static variant on SWE-Bench (5 runs); leftmost point is under the largest budget constraint.}
        \label{fig:slp_WRO}
    \end{minipage}
\end{figure}

\noindent\textbf{Performance comparison.}
\label{sec: offline_res} 
Table~\ref{tab:offline_performance_comparison} reports test performance across budget levels on RouterBench and SWE-Bench.          
\alp consistently matches or outperforms all baselines at low-to-mid budgets. 
FrugalGPT fails entirely at the lowest budgets on both datasets, as its cascade policy requires invoking multiple models per query and thus cannot satisfy tight per-query cost constraints. CascadeRouting is competitive on RouterBench, where it uses test-set statistics following its original setting, but drops on SWE-Bench where only regression-based estimates are available, revealing sensitivity to estimation quality.
At the highest budget, most methods converge as any model can be freely invoked.

\begin{figure*}[t]
\centering
\begin{subfigure}[t]{0.48\textwidth}
    \centering
\includegraphics[width=0.9\textwidth]{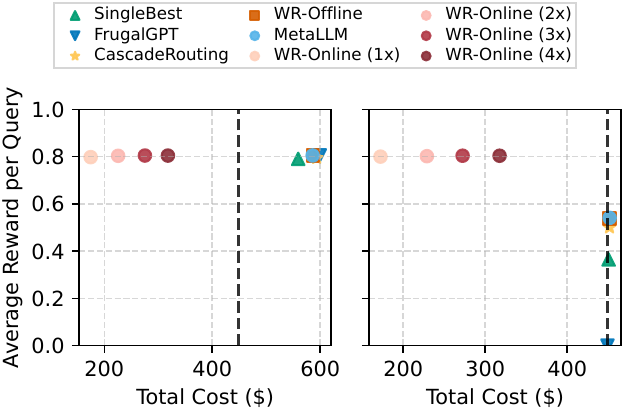}
    \caption{Left: generous budget. Right: constrained budget.}
    \label{fig:epsalp_all}
\end{subfigure}
\hfill
\begin{subfigure}[t]{0.48\textwidth}
    \centering
\includegraphics[width=\textwidth]{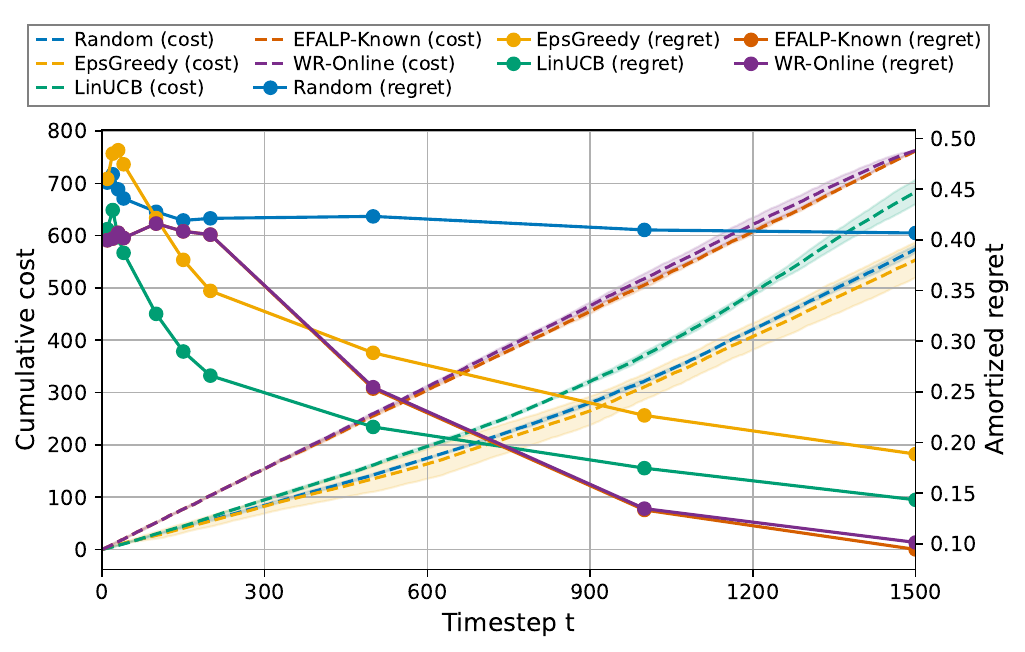}
    \caption{Simulation experiment 
    % under $B=762.58$.
    }
    \label{fig:simulation}
\end{subfigure}
% \vspace{-0.5em}
\caption{(a) \efalp vs.\ offline baselines on RouterBench; vertical dashed lines mark offline training cost. (b) Cumulative cost and amortized regret of budget-aware online algorithms (5 runs).}
\label{fig:routerbench-results}
% \vspace{-0.5em}
\end{figure*}

\noindent\textbf{Budget Utilization.} 
Figure~\ref{fig:combined_remain_budget} plots the ratio of remaining to total budget; positive values near zero indicate effective utilization without overspending, while negative values indicate violations. On RouterBench, \alp maintains this ratio closest to zero across all budget levels because its adaptive re-solving with $b/\tau$ at each step continuously corrects for cost consumption. 
FrugalGPT underutilizes the budget at mid budget levels and fails entirely at the tightest budget. CascadeRouting generally respects the budget with occasional violations at low-to-mid levels. On SWE-Bench, a medium-cost model dominates the cost-performance frontier, leaving some budget unused across methods. The exception is FrugalGPT at the highest budget, which exploits the extra slack to achieve the best performance on this dataset (see also Table~\ref{tab:offline_performance_comparison}).

\noindent\textbf{Latency Comparison.} 
Table~\ref{tab:latency_routerbench} reports training and per-query inference latency on RouterBench. FrugalGPT and CascadeRouting re-optimize cascade parameters \textit{per budget level}, incurring high training cost. \alp trains in $1.50$s on CPU, over $5{,}900\times$ faster than FrugalGPT which requires $8.85\times10^3$s on GPU because \alp estimates statistics once and adapts to any budget at inference without retraining. 

\noindent\textbf{Adaptive LP vs.\ Static LP constraint.} 
To assess whether \alp's adaptive constraint is essential for budget compliance, we compare it against a static variant that solves $\mathrm{LP}_{T,B}$ with the same estimated statistics. We focus on the more realistic SWE-Bench, where high per-query cost variance makes budget drift more likely and adaptivity more important. Each point in Figure~\ref{fig:slp_WRO} shows its budget consumption and performance difference w.r.t. \alp at one budget value. Without adaptive constraint, the static variant overshoots the workload budget by $>5\%$; such overspends may be unacceptable when compounded over long horizons or over repeated deployments. \alp avoids cost violation through adaptive LP, maintaining budget compliance without sacrificing reward.

\subsection{\efalp Experiment}
\label{sec:online_res}

\noindent\textbf{Performance Under End-to-End Budget.} 
We compare \efalp with offline methods under end-to-end budgets covering both data collection and deployment,  in two scenarios: \textit{generous} -- full training and maximum deployment budget and \textit{constrained} -- full training and minimum deployment budget (Figure~\ref{fig:epsalp_all}). In the generous setting, all offline methods converge to around $0.8$ average quality on RouterBench. \efalp reaches this with a single pass over training queries, reducing data collection cost by $90\%$. This suggests that marginal gains from additional supervision diminish rapidly. On SWE-Bench (Figure~\ref{fig:epsalp_combined_swebench}, Appendix~\ref{online-swe}), \efalp improves with more exploration, reflecting that this dataset requires longer exploration to accurately estimate reward statistics.
In the constrained setting, offline methods pre-commit training expense against deployment budget. \efalp reallocates unspent exploration budget to deployment, matching the generous-setting performance under the same end-to-end budget. Appendix~\ref{app:estimation_accuracy} validates that the exploration converges accurately to ground-truth statistics.

\noindent\textbf{Simulation Experiment for \efalp.}
To validate Theorem~\ref{thm:efalp-regret}, we run a controlled simulation under a global budget (Figure~\ref{fig:simulation}). \efalp's amortized regret decays faster than $O(1/\sqrt{T})$, consistent with Theorem~\ref{thm:efalp-regret}, and nearly matches $\epsilon$-first ALP with known costs, with the gap attributable to online cost estimation. 
Other algorithms under-spend significantly while \efalp tracks the $B/T$ slope closely, incurring lower regret. See Appendix~\ref{app:simulation} for detailed settings.

\section{Conclusions and Limitations}
\label{sec:conclue_limit}
We present \optiroute, a unified LLM routing framework that formulates routing as a constrained contextual bandit under a workload-level budget. \alp enforces budget via Adaptive Linear Programming and adapts to new budgets at inference without retraining; \efalp extends this to the data-scarce setting by jointly learning reward and cost statistics online under a common shared budget. We prove the end-to-end regret bound for \efalp covering exploration and exploitation, and demonstrate superior performance at tight budgets, precise budget utilization, and up to $90\%$ reduction in training data with \efalp.

\emph{Limitations.} 
The i.i.d.\ context assumption holds in stable API query workloads but may be violated when queries exhibit dependencies or distributional shift, in which case regret bound no longer applies; see more discussion in Appendix~\ref{app:assumptions}.
We only explored monetary cost as the constraint on two datasets; extending to other cost definitions and benchmarks with larger model pools is a natural direction for future work.

\bibliographystyle{plainnat}
\bibliography{references}

@inproceedings{kmeans1967,
  title={Some methods of classification and analysis of multivariate observations},
  author={McQueen, James B},
  booktitle={Proc. of 5th Berkeley Symposium on Math. Stat. and Prob.},
  pages={281--297},
  year={1967}
}

@INPROCEEDINGS{modelpruning1,
  author={Hassibi, B. and Stork, D.G. and Wolff, G.J.},
  booktitle={IEEE International Conference on Neural Networks}, 
  title={Optimal Brain Surgeon and general network pruning}, 
  year={1993},
  volume={},
  number={},
  pages={293-299 vol.1},
  doi={10.1109/ICNN.1993.298572}}

@inproceedings{modelpruning2,
  title = {Optimal Brain Damage},
  booktitle = {Advances in Neural Information Processing Systems},
  author = {LeCun, Yann and Denker, John and Solla, Sara},
  editor = {Touretzky, D.},
  year = 1989,
  volume = {2},
  publisher = {Morgan-Kaufmann}
}

@INPROCEEDINGS{quantization1,
  author={Jacob, Benoit and Kligys, Skirmantas and Chen, Bo and Zhu, Menglong and Tang, Matthew and Howard, Andrew and Adam, Hartwig and Kalenichenko, Dmitry},
  booktitle={2018 IEEE/CVF Conference on Computer Vision and Pattern Recognition}, 
  title={Quantization and Training of Neural Networks for Efficient Integer-Arithmetic-Only Inference}, 
  year={2018},
  volume={},
  number={},
  pages={2704-2713},
  doi={10.1109/CVPR.2018.00286}}

@inproceedings{quantization2,title	= {Improving the speed of neural networks on CPUs},author	= {Vincent Vanhoucke and Andrew Senior and Mark Z. Mao},year	= {2011},booktitle	= {Deep Learning and Unsupervised Feature Learning Workshop, NIPS 2011}}

@inproceedings{moe1,
  author       = {Noam Shazeer and
                  Azalia Mirhoseini and
                  Krzysztof Maziarz and
                  Andy Davis and
                  Quoc V. Le and
                  Geoffrey E. Hinton and
                  Jeff Dean},
  title        = {Outrageously Large Neural Networks: The Sparsely-Gated Mixture-of-Experts
                  Layer},
  booktitle    = {5th International Conference on Learning Representations, {ICLR} 2017,
                  Toulon, France, April 24-26, 2017, Conference Track Proceedings},
  publisher    = {OpenReview.net},
  year         = {2017},
  url          = {https://openreview.net/forum?id=B1ckMDqlg},
}

@article{moe2,
author = {Fedus, William and Zoph, Barret and Shazeer, Noam},
title = {Switch transformers: scaling to trillion parameter models with simple and efficient sparsity},
year = {2022},
issue_date = {January 2022},
publisher = {JMLR.org},
volume = {23},
number = {1},
issn = {1532-4435},
journal = {J. Mach. Learn. Res.},
month = jan,
articleno = {120},
numpages = {39}
}

@article{ucb,
author = {Auer, Peter and Cesa-Bianchi, Nicol\`{o} and Fischer, Paul},
title = {Finite-time Analysis of the Multiarmed Bandit Problem},
year = {2002},
issue_date = {May-June 2002},
publisher = {Kluwer Academic Publishers},
address = {USA},
volume = {47},
number = {2–3},
issn = {0885-6125},
url = {https://doi.org/10.1023/A:1013689704352},
doi = {10.1023/A:1013689704352},
journal = {Mach. Learn.},
month = may,
pages = {235–256},
numpages = {22}
}

@article{ef-budget-MAB, title={Epsilon–First Policies for Budget–Limited Multi-Armed Bandits}, volume={24}, DOI={10.1609/aaai.v24i1.7758},  number={1}, journal={Proceedings of the AAAI Conference on Artificial Intelligence}, author={Tran-Thanh, Long and Chapman, Archie and Munoz de Cote, Enrique and Rogers, Alex and Jennings, Nicholas R.}, year={2010}, month={Jul.}, pages={1211-1216} }

@inproceedings{budget-MAB-variable-cost,
author = {Ding, Wenkui and Qiny, Tao and Zhang, Xu-Dong and Liu, Tie-Yan},
title = {Multi-armed bandit with budget constraint and variable costs},
year = {2013},
publisher = {AAAI Press},
booktitle = {Proceedings of the Twenty-Seventh AAAI Conference on Artificial Intelligence},
pages = {232–238},
numpages = {7},
location = {Bellevue, Washington},
series = {AAAI'13}
}

@inproceedings{ucb-alp,
author = {Wu, Huasen and Srikant, R. and Liu, Xin and Jiang, Chong},
title = {Algorithms with logarithmic or sublinear regret for constrained contextual bandits},
year = {2015},
publisher = {MIT Press},
address = {Cambridge, MA, USA},
booktitle = {Proceedings of the 29th International Conference on Neural Information Processing Systems - Volume 1},
pages = {433–441},
numpages = {9},
location = {Montreal, Canada},
series = {NIPS'15}
}

@book{SuttonRL,
author = {Sutton, Richard S. and Barto, Andrew G.},
title = {Reinforcement Learning: An Introduction},
year = {2018},
isbn = {0262039249},
publisher = {A Bradford Book},
address = {Cambridge, MA, USA},
}

@InProceedings{resourcefulbandit,
  title = 	 {Resourceful Contextual Bandits},
  author = 	 {Badanidiyuru, Ashwinkumar and Langford, John and Slivkins, Aleksandrs},
  booktitle = 	 {Proceedings of The 27th Conference on Learning Theory},
  pages = 	 {1109--1134},
  year = 	 {2014},
  editor = 	 {Balcan, Maria Florina and Feldman, Vitaly and Szepesvári, Csaba},
  volume = 	 {35},
  series = 	 {Proceedings of Machine Learning Research},
  address = 	 {Barcelona, Spain},
  month = 	 {13--15 Jun},
  publisher =    {PMLR},	 
}

@inproceedings{linucb,
author = {Li, Lihong and Chu, Wei and Langford, John and Schapire, Robert E.},
title = {A contextual-bandit approach to personalized news article recommendation},
year = {2010},
isbn = {9781605587998},
publisher = {Association for Computing Machinery},
address = {New York, NY, USA},
doi = {10.1145/1772690.1772758},
booktitle = {Proceedings of the 19th International Conference on World Wide Web},
pages = {661–670},
numpages = {10},
location = {Raleigh, North Carolina, USA},
series = {WWW '10}
}

@article{HELM23,
  title = {Holistic Evaluation of Language Models},
  author = {Liang, Percy and Bommasani, Rishi and Lee, Tony and Tsipras, Dimitris and Soylu, Dilara and Yasunaga, Michihiro and Zhang, Yian and Narayanan, Deepak and Wu, Yuhuai and Kumar, Ananya and Newman, Benjamin and Yuan, Binhang and Yan, Bobby and Zhang, Ce and Cosgrove, Christian and Manning, Christopher D and Re, Christopher and {Acosta-Navas}, Diana and Hudson, Drew A. and Zelikman, Eric and Durmus, Esin and Ladhak, Faisal and Rong, Frieda and Ren, Hongyu and Yao, Huaxiu and WANG, Jue and Santhanam, Keshav and Orr, Laurel and Zheng, Lucia and Yuksekgonul, Mert and Suzgun, Mirac and Kim, Nathan and Guha, Neel and Chatterji, Niladri S. and Khattab, Omar and Henderson, Peter and Huang, Qian and Chi, Ryan Andrew and Xie, Sang Michael and Santurkar, Shibani and Ganguli, Surya and Hashimoto, Tatsunori and Icard, Thomas and Zhang, Tianyi and Chaudhary, Vishrav and Wang, William and Li, Xuechen and Mai, Yifan and Zhang, Yuhui and Koreeda, Yuta},
  year = 2023,
  journal = {Transactions on Machine Learning Research},
  issn = {2835-8856}
}

@inproceedings{s3,
author = {Jin, Yunho and Wu, Chun-Feng and Brooks, David and Wei, Gu-Yeon},
title = {S3: increasing GPU utilization during generative inference for higher throughput},
year = {2023},
publisher = {Curran Associates Inc.},
address = {Red Hook, NY, USA},
booktitle = {Proceedings of the 37th International Conference on Neural Information Processing Systems},
articleno = {791},
numpages = {13},
location = {New Orleans, LA, USA},
series = {NIPS '23}
}

@inproceedings{speculativedecode,
  title={Fast inference from transformers via speculative decoding},
  author={Leviathan, Yaniv and Kalman, Matan and Matias, Yossi},
  booktitle={International Conference on Machine Learning},
  pages={19274--19286},
  year={2023},
  organization={PMLR}
}

@article{speculativedecode2,
  title={Accelerating large language model decoding with speculative sampling},
  author={Chen, Charlie and Borgeaud, Sebastian and Irving, Geoffrey and Lespiau, Jean-Baptiste and Sifre, Laurent and Jumper, John},
  journal={arXiv preprint arXiv:2302.01318},
  year={2023}
}

@misc{metallm,
      title={MetaLLM: A High-performant and Cost-efficient Dynamic Framework for Wrapping LLMs}, 
      author={Quang H. Nguyen and Thinh Dao and Duy C. Hoang and Juliette Decugis and Saurav Manchanda and Nitesh V. Chawla and Khoa D. Doan},
      year={2025},
      eprint={2407.10834},
      archivePrefix={arXiv},
      primaryClass={cs.LG},
      url={https://arxiv.org/abs/2407.10834}, 
}

@inproceedings{
hybridllm,
title={Hybrid {LLM}: Cost-Efficient and Quality-Aware Query Routing},
author={Dujian Ding and Ankur Mallick and Chi Wang and Robert Sim and Subhabrata Mukherjee and Victor R{\"u}hle and Laks V. S. Lakshmanan and Ahmed Hassan Awadallah},
booktitle={The Twelfth International Conference on Learning Representations},
year={2024},
url={https://openreview.net/forum?id=02f3mUtqnM}
}

@article{routerbench,
  author       = {Qitian Jason Hu and
                  Jacob Bieker and
                  Xiuyu Li and
                  Nan Jiang and
                  Benjamin Keigwin and
                  Gaurav Ranganath and
                  Kurt Keutzer and
                  Shriyash Kaustubh Upadhyay},
  title        = {RouterBench: {A} Benchmark for Multi-LLM Routing System},
  journal      = {CoRR},
  volume       = {abs/2403.12031},
  year         = {2024},
  url          = {https://doi.org/10.48550/arXiv.2403.12031},
  doi          = {10.48550/ARXIV.2403.12031},
  eprinttype    = {arXiv},
  eprint       = {2403.12031},
}

@inproceedings{routellm,
title={Route{LLM}: Learning to Route {LLM}s from Preference Data},
author={Isaac Ong and Amjad Almahairi and Vincent Wu and Wei-Lin Chiang and Tianhao Wu and Joseph E. Gonzalez and M Waleed Kadous and Ion Stoica},
booktitle={The Thirteenth International Conference on Learning Representations},
year={2025},
url={https://openreview.net/forum?id=8sSqNntaMr}
}

@inproceedings{mixLLM2025,
  title = {{{MixLLM}}: {{Dynamic Routing}} in {{Mixed Large Language Models}}},
  shorttitle = {{{MixLLM}}},
  booktitle = {Proceedings of the 2025 {{Conference}} of the {{Nations}} of the {{Americas Chapter}} of the {{Association}} for {{Computational Linguistics}}: {{Human Language Technologies}} ({{Volume}} 1: {{Long Papers}})},
  author = {Wang, Xinyuan and Liu, Yanchi and Cheng, Wei and Zhao, Xujiang and Chen, Zhengzhang and Yu, Wenchao and Fu, Yanjie and Chen, Haifeng},
  editor = {Chiruzzo, Luis and Ritter, Alan and Wang, Lu},
  year = {2025},
  month = apr,
  pages = {10912--10922},
  publisher = {Association for Computational Linguistics},
  address = {Albuquerque, New Mexico},
  urldate = {2025-05-13},
  isbn = {979-8-89176-189-6},
}

@inproceedings{OnlineMultiLLMSelection2025,
  title={{Online Multi-LLM Selection via Contextual Bandits under Unstructured Context Evolution}},
  author={Poon, Manhin and Dai, XiangXiang and Liu, Xutong and Kong, Fang and Lui, John C.S. and Zuo, Jinhang},
  booktitle={Proceedings of the AAAI Conference on Artificial Intelligence},
  volume={40},
  year={2026}
}

@article{frugalgpt2024,
  title = {{{FrugalGPT}}: {{How}} to Use Large Language Models While Reducing Cost and Improving Performance},
  author = {Chen, Lingjiao and Zaharia, Matei and Zou, James},
  year = {2024},
  journal = {Transactions on Machine Learning Research},
  issn = {2835-8856}
}

@inproceedings{
UnifiedApproachRouting2025,
title={A Unified Approach to Routing and Cascading for {LLM}s},
author={Jasper Dekoninck and Maximilian Baader and Martin Vechev},
booktitle={Forty-second International Conference on Machine Learning},
year={2025},
url={https://openreview.net/forum?id=AAl89VNNy1}
}

@inproceedings{deepen,
  title={Ensemble Learning for Heterogeneous Large Language Models with Deep Parallel Collaboration},
  author={Yi-Chong Huang and Xiaocheng Feng and Baohang Li and Yang Xiang and Hui Wang and Bing Qin and Ting Liu},
  booktitle={Neural Information Processing Systems},
  year={2024},
  url={https://api.semanticscholar.org/CorpusID:269282634}
}

@article{thriftllm,
author = {Huang, Keke and Shi, Yimin and Ding, Dujian and Li, Yifei and Fei, Yang and Lakshmanan, Laks and Xiao, Xiaokui},
title = {ThriftLLM: On Cost-Effective Selection of Large Language Models for Classification Queries},
year = {2025},
issue_date = {July 2025},
publisher = {VLDB Endowment},
volume = {18},
number = {11},
issn = {2150-8097},
url = {https://doi.org/10.14778/3749646.3749702},
doi = {10.14778/3749646.3749702},
journal = {Proc. VLDB Endow.},
month = sep,
pages = {4410–4423},
numpages = {14}
}

@inproceedings{
pyramidkv,
title={Pyramid{KV}: Dynamic {KV} Cache Compression based on Pyramidal Information Funneling},
author={Zefan Cai and Yichi Zhang and Bofei Gao and Yuliang Liu and Yucheng Li and Tianyu Liu and Keming Lu and Wayne Xiong and Yue Dong and Junjie Hu and Wen Xiao},
booktitle={Second Conference on Language Modeling},
year={2025},
url={https://openreview.net/forum?id=ayi7qezU87}
}

@inproceedings{llmblender,
    title = "{LLM}-Blender: Ensembling Large Language Models with Pairwise Ranking and Generative Fusion",
    author = "Jiang, Dongfu  and
      Ren, Xiang  and
      Lin, Bill Yuchen",
    editor = "Rogers, Anna  and
      Boyd-Graber, Jordan  and
      Okazaki, Naoaki",
    booktitle = "Proceedings of the 61st Annual Meeting of the Association for Computational Linguistics (Volume 1: Long Papers)",
    month = jul,
    year = "2023",
    address = "Toronto, Canada",
    publisher = "Association for Computational Linguistics",
    doi = "10.18653/v1/2023.acl-long.792",
    pages = "14165--14178",
}

@inproceedings{zooter,
    title = "Routing to the Expert: Efficient Reward-guided Ensemble of Large Language Models",
    author = "Lu, Keming  and
      Yuan, Hongyi  and
      Lin, Runji  and
      Lin, Junyang  and
      Yuan, Zheng  and
      Zhou, Chang  and
      Zhou, Jingren",
    editor = "Duh, Kevin  and
      Gomez, Helena  and
      Bethard, Steven",
    booktitle = "Proceedings of the 2024 Conference of the North American Chapter of the Association for Computational Linguistics: Human Language Technologies (Volume 1: Long Papers)",
    month = jun,
    year = "2024",
    address = "Mexico City, Mexico",
    publisher = "Association for Computational Linguistics",
    doi = "10.18653/v1/2024.naacl-long.109",
    pages = "1964--1974",
}

@inproceedings{
    swebench,
    title={{SWE}-bench: Can Language Models Resolve Real-world Github Issues?},
    author={Carlos E Jimenez and John Yang and Alexander Wettig and Shunyu Yao and Kexin Pei and Ofir Press and Karthik R Narasimhan},
    booktitle={The Twelfth International Conference on Learning Representations},
    year={2024},
    url={https://openreview.net/forum?id=VTF8yNQM66}
}

@inproceedings{chatbotArena24,
author = {Chiang, Wei-Lin and Zheng, Lianmin and Sheng, Ying and Angelopoulos, Anastasios N. and Li, Tianle and Li, Dacheng and Zhu, Banghua and Zhang, Hao and Jordan, Michael I. and Gonzalez, Joseph E. and Stoica, Ion},
title = {Chatbot arena: an open platform for evaluating LLMs by human preference},
year = {2024},
publisher = {JMLR.org},
booktitle = {Proceedings of the 41st International Conference on Machine Learning},
articleno = {331},
numpages = {30},
location = {Vienna, Austria},
series = {ICML'24}
}

@inproceedings{bestroute,
title={{BEST}-Route: Adaptive {LLM} Routing with Test-Time Optimal Compute},
author={Dujian Ding and Ankur Mallick and Shaokun Zhang and Chi Wang and Daniel Madrigal and Mirian Del Carmen Hipolito Garcia and Menglin Xia and Laks V. S. Lakshmanan and Qingyun Wu and Victor R{\"u}hle},
booktitle={Forty-second International Conference on Machine Learning},
year={2025},
url={https://openreview.net/forum?id=tFBIbCVXkG}
}

@article{omnirouter25,
author = {Mei, Kai and Xu, Wujiang and Guo, Minghao and Lin, Shuhang and Zhang, Yongfeng},
title = {OmniRouter: Budget and Performance Controllable Multi-LLM Routing},
year = {2026},
issue_date = {December 2025},
publisher = {Association for Computing Machinery},
address = {New York, NY, USA},
volume = {27},
number = {2},
issn = {1931-0145},
url = {https://doi.org/10.1145/3787470.3787480},
doi = {10.1145/3787470.3787480},
journal = {SIGKDD Explor. Newsl.},
month = dec,
pages = {107–116},
numpages = {10}
}

@inproceedings{cascadeMixoT24,
title={Large Language Model Cascades with Mixture of Thought Representations for Cost-Efficient Reasoning},
author={Murong Yue and Jie Zhao and Min Zhang and Liang Du and Ziyu Yao},
booktitle={The Twelfth International Conference on Learning Representations},
year={2024},
url={https://openreview.net/forum?id=6okaSfANzh}
}

@inproceedings{cascadeToken24,
title={Language Model Cascades: Token-Level Uncertainty And Beyond},
author={Neha Gupta and Harikrishna Narasimhan and Wittawat Jitkrittum and Ankit Singh Rawat and Aditya Krishna Menon and Sanjiv Kumar},
booktitle={The Twelfth International Conference on Learning Representations},
year={2024},
url={https://openreview.net/forum?id=KgaBScZ4VI}
}

@inproceedings{pilot2025,
    title = "Adaptive {LLM} Routing under Budget Constraints",
    author = "Panda, Pranoy  and
      Magazine, Raghav  and
      Devaguptapu, Chaitanya  and
      Takemori, Sho  and
      Sharma, Vishal",
    editor = "Christodoulopoulos, Christos  and
      Chakraborty, Tanmoy  and
      Rose, Carolyn  and
      Peng, Violet",
    booktitle = "Findings of the Association for Computational Linguistics: EMNLP 2025",
    month = nov,
    year = "2025",
    address = "Suzhou, China",
    publisher = "Association for Computational Linguistics",
    doi = "10.18653/v1/2025.findings-emnlp.1301",
    pages = "23934--23949",
    ISBN = "979-8-89176-335-7",
}

@inproceedings{livebench25,
  title={LiveBench: A Challenging, Contamination-Free {LLM} Benchmark},
  author={Colin White and Samuel Dooley and Manley Roberts and Arka Pal and Benjamin Feuer and Siddhartha Jain and Ravid Shwartz-Ziv and Neel Jain and Khalid Saifullah and Sreemanti Dey and Shubh-Agrawal and Sandeep Singh Sandha and Siddartha Venkat Naidu and Chinmay Hegde and Yann LeCun and Tom Goldstein and Willie Neiswanger and Micah Goldblum},
  booktitle={The Thirteenth International Conference on Learning Representations},
  year={2025},
}

@misc{claude_team_pricing,
  author       = {{Anthropic}},
  title        = {{Claude Team}},
  howpublished = {\url{https://claude.com/pricing/team}},
  note         = {Accessed: 2026-04-28},
  year         = {2026}
}

@misc{openai_api_pricing,
  author       = {{OpenAI}},
  title        = {{API Pricing}},
  howpublished = {\url{https://openai.com/api/pricing/}},
  note         = {Accessed: 2026-04-28},
  year         = {2026}
}

%%%%%%%%%%%%%%%%%%%%%%%%%%%%%%%%%%%%%%%%%%%%%%%%%%%%%%%%%%%%
\newpage
\appendix

\onecolumn
\newpage
\addcontentsline{toc}{section}{Appendix}

\startcontents[app]
\clearpage
\phantomsection

\printcontents[app]{l}{1}{\section*{Appendix contents}}

\newpage
%% ============================================================
\section{Discussion of Assumptions}
\label{app:assumptions} 
We elaborate on the three assumptions introduced in Section~\ref{subsec:formulation} and discuss their practical validity.

\paragraph{I.i.d.\ contexts.} 
There are two distinct aspects of this assumption. First, \emph{query independence}: queries arrive without state transitions between them, enabling the contextual bandit formulation. This holds in the majority of real-world LLM API usage scenarios (e.g., document summarization, customer support all treat each query as a self-contained request), and is implicit in all supervised routing baselines, which also treat queries independently. Second, \emph{stationarity}: the context distribution $\{\pi_j\}$ is assumed stable within the planning horizon $T$. Setting $T$ to a window that captures recurring traffic patterns (e.g., weekly business-hours cycles) makes this reasonable in practice. For longer-term drift, the context distribution and LP can be easily updated from recent traffic without re-exploring model statistics.
% Queries are drawn independently from a stationary distribution $\{\pi_j\}$, estimated via clustering from a historical set (Section~\ref{subsec: clustering}). This holds when queries arrive from a stable user population. In settings with distributional shift, the centroids and context distribution can be re-estimated from recent query history and the LP re-solved without retraining the encoder.

\paragraph{Finite horizon $T$.} 
A finite $T$ is standard in the constrained contextual bandit literature as it makes regret analysis tractable; $T$ can be set arbitrarily large, and the analysis characterizes how regret grows as a function of $T$. In practice, $T$ is calibrated to the deployment window (e.g., a monthly API quota) using historical query volume estimates.

% Both \alp and \efalp require $T$ to compute the per-round budget ratio $b/\tau$. In practice, $T$ corresponds to a planning window such as a monthly API quota and can be estimated from historical traffic logs.

\paragraph{Feedback acquisition.} 
The monetary cost $S_t$ is available immediately via token-count pricing. The reward $Y_t$ is not restricted to ground-truth correctness: any observable scalar signal like proxy reward model scores, user thumbs-up/down feedback, or code accept rates can qualify. In \alp, $Y_t$ is read from $\mathcal{D}_{\mathrm{hist}}$ and is not needed at inference time. In \efalp, $Y_t$ is used during exploration to update $\bar{u}_{j,k}$; per-cluster averaging over queries naturally absorbs observation noise, and delayed reward updates can be deferred without interrupting routing decisions. 
% The monetary cost \(S_t\) is available as soon as generation completes via token-count pricing. In \alp, the quality score $Y_t$ is read from $\mathcal{D}_{\mathrm{hist}}$ and is not needed at inference time. In \efalp, $Y_t$ is used during exploration to update $\bar{u}_{j,k}$; however, it does not block the routing decision itself, which depends only on the visit counts $N_{j,k}$. For automated evaluations (e.g., unit-test pass rates), $Y_t$ is available immediately; in settings with delayed feedback, estimate updates can be deferred without interrupting routing.
\section{Related Work}
\label{app:related_work}

We expand on the related work highlights in Section~\ref{sec:related_work} with inference-time LLM selection and constrained bandit methods.

\paragraph{Inference-time LLM selection.} Motivated by the notable cost quality heterogeneity across today’s LLMs, recent work on efficient LLM inference has developed into three paradigms: routing, cascading and ensembling. 

\textbf{LLM Routing} aims to learn a policy that routes a query to the most suitable model from a pool of candidate LLMs. HybridLLM~\cite{hybridllm} and RouteLLM~\cite{routellm} focus on a two-model setting, train a router that preserves quality while reducing cost by sending easy queries to a smaller model. Subsequent work such as Zooter~\cite{zooter} and MixLLM~\cite{mixLLM2025} extends this supervise-trained router idea to larger model pools and explicitly manages the quality-cost trade-off. OmniRouter~\cite{omnirouter25} predicts cost and performance for each model-query pair using a retrieval-augmented predictor, then uses Lagrangian dual optimization to minimize global cost subject to a global performance threshold. This predictor requires dense historical model-query pair statistics, collecting which can be expensive. BestRoute~\cite{bestroute} explores best-of-$n$ sampling to better exploit test-time compute with smaller models and shows that this leads to significant reduction in cost compared to baselines. The approach is essentially a heuristic and lacks formal guarantees. Our work focuses on the routing setting where one model is called once per query with an explicit workload-level budget, and applies online learning with exploration when dense supervision dataset is too expensive to collect.

\textbf{LLM Cascading} method~\cite{frugalgpt2024, cascadeMixoT24, cascadeToken24} executes a pre-defined LLM sequence: a small, cheap model answers first and the system escalates to larger ones  when a learned threshold criterion fails; %recent unified formulations 
Dekoninck et al. \cite{UnifiedApproachRouting2025} cast routing and cascading within one framework and enforce explicit per-query budget. While they derive optimality results, per-query constraints can be myopic and suboptimal under a workload-level budget. 

\textbf{LLM Ensemble} methods such as DeePEn~\cite{deepen} aligns model logits in a shared space and aggregates them at the token level, while LLMBlender~\cite{llmblender} ensembles post hoc by training a T5-based model to fuse multiple model outputs into a refined response. These methods primarily optimize accuracy without explicit cost or latency constraints. ThriftLLM~\cite{thriftllm}  enforces per-query budget %constraints 
and achieves cost-efficient classification. Given our focus on single-model LLM routing with workload-level budget, we do not compare with these methods further. 
\eat{DeePEn~\cite{deepen} aligns model logits in a shared space and aggregates them at the token level, while LLMBlender~\cite{llmblender} ensembles post hoc by training a T5-based model to fuse multiple model outputs into a refined response. These methods primarily optimize accuracy without explicit cost or latency constraints.
In contrast, ThriftLLM~\cite{thriftllm} enforces per-query budget constraints, %using weighted aggregation to 
and achieves cost-efficient classification. However, these methods often call  more than one model per query and introduce additional thresholding check or aggregation step which can  increase the overall latency.}

\paragraph{Constrained Contextual Bandits.}
Contextual bandits extends the well-known multi-armed bandits (MAB)  by utilizing additional context at each round, allowing the agent to choose an action conditioned on the context. Common baselines include $\epsilon$-greedy~\cite{SuttonRL}, 
UCB~\cite{ucb}, and LinUCB~\cite{linucb}. 

In a constrained bandit setting, an agent optimizes the cumulative reward subject to a global budget constraint. A well-known strategy for budget-constrained MAB without contexts is $\epsilon$-first ~\cite{ef-budget-MAB}. \citet{resourcefulbandit} studies budget constrained contextual bandits and establishes an $\tilde O(\sqrt{T})$ regret bound; however, their approach is  computationally expensive in practice. Following this line, \cite{ucb-alp} proposes UCB-ALP for unit-cost systems where all arms have a uniform cost and $\epsilon$-first ALP for heterogeneous-cost systems, though rely on cost statistics being known a priori. Another line of work studies stochastic costs under budget constrained MAB using UCB-style frameworks~\cite{budget-MAB-variable-cost}. Our work extends $\epsilon$-first ALP to the regime where both reward and cost must be estimated online under a single shared budget covering exploration and exploitation.

\paragraph{LLM Routing as a Bandit Problem.}
Several recent studies explicitly formulate LLM routing as a contextual bandits problem: each candidate model is an action, query embeddings provide context, and reward is reflected by task performance or human preference. MetaLLM~\cite{metallm} uses UCB to select an LLM per query, leveraging reward estimates learned from offline preference data to balance cost against performance. MixLLM~\cite{mixLLM2025} integrates contextual bandits and selects a model by a LinUCB-like score augmented by uncertainty and a latency penalty. \citet{OnlineMultiLLMSelection2025} applies Greedy LinUCB to the context evolved from multi-step chat interaction under global budget. 
PILOT~\cite{PILOT2025} initializes a LinUCB router from preference-aligned embeddings and enforces budget at deployment via an online knapsack solver; its regret analysis covers only the router learning phase, not the full exploration-to-deployment pipeline. All these works either apply budget constraints only at deployment or lack end-to-end regret guarantees jointly covering learning and routing.
\section{Proof of $\widehat{U}(T,B)$ is Upper Bound}
\label{app:upperbound}

\begin{lemma}
\label{lemma:upperbound}
For a heterogeneous cost system with $K$ models, $J$ discrete contexts and known true statistics, if the time-horizon is $T$ and the budget is $B$, then the total expected reward $\widehat{U}(T,B) \geq U^*(T,B)$, where $\hat{U}(T,B) = T v(\rho)$.
\end{lemma}

\begin{proof}
Let $N_{j,k}$ be the number of rounds that an action $k$ is taken under context $j$ for any realization under any feasible algorithm with known statistics. Let 
$p_{j,k} := \mathbb{E}[N_{j,k}]/(\pi_j T)$, which satisfies $0 \le p_{j,k} \le 1$. 
Then the expected total reward for this feasible algorithm becomes 
$$
\sum_{j=1}^J \sum_{k=1}^K u_{j,k} \mathbb{E}[N_{j,k}] = T \sum_{j=1}^J \sum_{k=1}^K p_{j,k} \pi_j u_{j,k}.
$$

% $\sum_{j=1}^J \sum_{k=1}^K N_{j,k} c_{j,k} \le B$
Moreover, because we start with a feasible algorithm, the expected budget constraint is met for all realizations, i.e., $\mathbb{E}\left[\sum_{t=1}^T S_t\right] \le B$, we have 
\[
\mathbb{E}\left[\sum_{t=1}^T S_t\right]
= \sum_{j=1}^J \sum_{k=1}^K \mathbb{E}\!\left[ \sum_{t=1}^T S_t \mathbbm{1}\{X_t=j,A_t=k\}\right]
= \sum_{j,k} c_{j,k}\,\mathbb{E}[N_{j,k}]
= T \sum_j \pi_j \sum_k p_{j,k} c_{j,k}.
\]
Feasibility in expectation implies
\[
\sum_j \pi_j \sum_k p_{j,k} c_{j,k} \le B/T.
\]

Thus, the expected total reward obtained by any feasible algorithm, including the oracle algorithm, is upper bounded by $\widehat{U}(T,B)$.
\end{proof}
\section{Detailed Regret Bound of ALP}
\label{app:detial_alp_regret}

To derive this upper bound, we need to derive the explicit solution of the upper bounding linear program $\mathrm{LP}_{T,B}$ and analyze the regret incurred when $b/\tau$ deviates from $B/T$. The details of this derivation, following the approach in~\cite{ucb-alp}, are provided in Appendix~\ref{app:solution}. Here, we introduce the essential concepts and quantities that will be used in the main regret bound of Theorem~\ref{thm:alp_regret}.

\textbf{Deriving explicit LP solution.}
The goal is to rank all pairs globally by their reward rates, then apply greedy algorithm to prioritize pairs with the highest reward rates and select those actions with corresponding probabilities. 

We first prune the candidate sets of context-action pairs using two measurements. The reward rate $\eta_{j,k} = \frac{u_{j,k}}{c_{j,k}}$ measures expected utility per unit cost for pair $(j,k)$. The ratio
$\xi_{j,k_1,k_2}
:=\frac{u_{j,k_1}-u_{j,k_2}}{c_{j,k_1}-c_{j,k_2}}, k_1\neq k_2,$ measures the additional expected reward per additional unit cost when switching from $k_2$ to $k_1$ under context $j$.
We can remove certain suboptimal pairs without decreasing the LP optimum by comparing $\eta_{j,k}$ and $\xi_{j, k_1, k_2}$ (exact conditions in Appendix~\ref{app:solution}). After pruning, the remaining actions within each context admit a consistent ordering w.r.t. $\xi_{j,k_1,k_2}$. 

We then define a transformation detailed in Appendix~\ref{app:solution} to obtain an equivalent linear program $\mathrm{\widetilde{LP}}_{T,B}$ such that the remaining candidate pairs can be globally sorted by a transformed marginal reward ratio $\tilde \eta_{j,k}$. The key requirement for this reduction is that for any context $j$, the ordered ratios $\xi_{j,k_1,k_2}$
lead to a consistent ranking of $\tilde\eta_{j,k}$. Let the pairs sorted in descending order of $\tilde\eta_{j,k}$ be $\{(j^{(1)},k^{(1)}),\dots,(j^{(M)},k^{(M)})\}$, where $M$ is the total number of remaining pairs after pruning. We  denote $\tilde\eta_1:=\max_i \tilde\eta_{j^{(i)},k^{(i)}}$ and $\tilde\eta_M:=\min_i \tilde\eta_{j^{(i)},k^{(i)}}$ accordingly.

\textbf{Boundary and non-boundary cases.}
The performance of ALP depends on how the average remaining budget $\frac{b}{\tau}$ evolves relative to the static $B/T$. In particular, the regret is affected by how often $b/\tau$ crosses certain critical thresholds associated with the sorted transformed pairs.

We define the cumulative expected cost mass up to index $i$ as $Q_i:=\sum_{i'=1}^{i} \pi_{j^{(i')}}\,\tilde c_{j^{(i')},k^{(i')}}$, representing the total expected cost for the first $i$ transformed items in the ordering $M$. For any static per-round budget $\rho\in(0,Q]$, we define the threshold index $\tilde i(\rho):=\max\{i\in\{0,1,\dots,M\}: Q_i\le \rho\,\}.$ We say that $\rho$ is \textit{non-boundary} if $\rho\neq Q_i$ for all $i\in\{1,\dots,M\}$ and \textit{boundary} if $\rho = Q_{i(\rho)}$ for some $ i(\rho)$.

\textbf{Concentration of $b$.}
 Since regret mainly occurs when the empirical average remaining budget $b/\tau$ deviates from $\rho$, we apply Lemma~12 of~\cite{ucb-alp} based on the Azuma--Hoeffding inequality to obtain the tail bound $\mathbb{P}\{|b/\tau-\rho|<\delta\}\geq e^{-\kappa\delta^2\tau}, 0<\delta<1$
for some constant $\kappa = \frac{\rho^2}{2}$, which controls the probability of large deviations.

The following theorem states that the ALP algorithm has a $O(1)$ regret in the non-boundary case, and incurs an $O(\sqrt{T})$ regret in the boundary case.
\begin{theorem}
\label{thm:alp_regret}

Given any fixed $\rho \in (0,1)$, there exists a positive constant $\kappa := \frac{\rho^2}{2}$, such that the regret of ALP satisfies:

\begin{enumerate}
    \item[(1)]  \textbf{(Non-boundary cases)} If $\rho \neq Q_{i}$ for any $i \in \{1,2,\ldots,M\}$, then
    $
    R_{\text{ALP}}(T,B) \leq \frac{\tilde{\eta}_1 - \tilde{\eta}_M}{1 - e^{-2\kappa \delta^2}},
    $
    where $\delta = \min\{\rho - Q_{\tilde{i}(\rho)}, Q_{\tilde{i}(\rho)+1} - \rho\}$.
    
    \item[(2)]  \textbf{(Boundary cases)} If $\rho = Q_{i(\rho)}$ for some $i \in \{1,2,\ldots,M\}$, then
    $
    R_{\text{ALP}}(T,B) \leq \Theta^{(\circ)} \sqrt{T} + \frac{\tilde{\eta}_1 - \tilde{\eta}_M}{1 - e^{-2\kappa \delta^2}},
    $
    where $\Theta^{(\circ)} = (\tilde{\eta}_1 - \tilde{\eta}_M) \sum^T_{\tau=1} \sqrt\frac{\pi}{\kappa\tau}$ and 
    $\delta = \min\{\rho - Q_{\tilde{i}(\rho)-1}, Q_{\tilde{i}(\rho)+1} - \rho\}$.
\end{enumerate}
\end{theorem}

\textbf{Proof Sketch.} The proof for Theorem~\ref{thm:alp_regret} follows a similar structure as the unit-cost ALP regret proof in Appendix B.2 from~\cite{ucb-alp}. For the non-boundary cases, the single-round expected reward satisfies $\mathbb{E}[v(b/\tau)] = v(\rho)$ if the threshold $\tilde{i}(b/\tau) = \tilde{i}(\rho)$ for all possible vaules of $b$, where $v(\rho)$ is the optimal single-round value of $\mathrm{LP}_{T,B}$ at average budget $\rho = B/T$. The regret is then bounded by a constant because the probability of the event $\tilde{i}(b / \tau) \neq \tilde{i}(\rho)$ decays exponentially due to the concentration bound of $b$. 
For the boundary cases, we show the theorem by expressing the regret in terms of the expected absolute deviation of the average remaining budget  $\frac{b}{\tau}$ from the target value $\rho$, i.e., $\mathbb{E}[|\frac{b}{\tau} - \rho|]$. \qed 
\section{Solution Formulation for $\mathrm{LP}_{T,B}$}
\label{app:solution}
In this section, we first decompose $\mathrm{LP}_{T,B}$ into subproblem $\mathrm{SP}_j$ which is a constrained problem on a single context $j$. Secondly, we show some actions can be deleted without affecting the performance in $\mathrm{LP}_{T,B}$.

\begin{lemma}
\label{lemma: remove_action}
For any given $\rho_j \ge 0$, there exists an optimal solution of $\mathrm{SP}_j$, 
i.e., $\mathbf{p}_j^* = (p_{j,1}^*, p_{j,2}^*, \ldots, p_{j,K_j}^*)$, which satisfies:
\begin{enumerate}
    \item For $k_1$, if there exists another action $k_2$ such that 
    $\eta_{j,k_1} \le \eta_{j,k_2}$ and $u_{j,k_1} \le u_{j,k_2}$, 
    then $p_{j,k_1}^* = 0$;
    \item For $k_1$, if there exist two actions $k_2$ and $k_3$ such that 
    $\eta_{j,k_2} \le \eta_{j,k_1} \le \eta_{j,k_3}$, 
    $u_{j,k_2} \ge u_{j,k_1} \ge u_{j,k_3}$, and 
    $
    \frac{u_{j,k_1} - u_{j,k_3}}{c_{j,k_1} - c_{j,k_3}}
    \le 
    \frac{u_{j,k_2} - u_{j,k_3}}{c_{j,k_2} - c_{j,k_3}},
    $
    then $p_{j,k_1}^* = 0$.
\end{enumerate}
\end{lemma}

Intuitively, the first part of Lemma~\ref{lemma: remove_action} shows that if an action has small normalized and original expected reward, then it can be removed. 
The second part of Lemma~\ref{lemma: remove_action} shows that if an action has small normalized expected reward and medium original expected reward, but the increasing rate is smaller than another action with larger expected reward, then it can also be removed. Detailed proof can be found in~\cite{ucb-alp}

With probabilities $p_{j,k}\in[0,1]$ of choosing action $k$ under context $j$, the oracle upper bound
is the optimal value $\hat v(\rho)$ of
\begin{align}
\max_{\{p_{j,k}\}}~~
&\sum_{j=1}^{J}\pi_j\sum_{k}p_{j,k}\,u_{j,k}
\label{eq:hetero-lp-obj}\\
\text{s.t.}\quad
&\sum_{j=1}^{J}\pi_j\sum_{k}p_{j,k}\,c_{j,k}\le \rho,\qquad\\
&\sum_{k}p_{j,k}\le 1~(\forall j),\qquad \\
&p_{j,k}\in[0,1].
\label{eq:hetero-lp-constraints}
\end{align}
The intra-context constraint $\sum_k p_{j,k}\le 1$ enforces the probabilities of taking action $k$ under context $j$ forms a probability distribution.

The constraints can be decoupled by first allocating budget for each context, and then solving a subproblem with the allocated budget constraint for each context. Specifically, let $\rho_j$ be the budget allocated to context $j$, then $\mathcal{LP}_{T,B}$ can be decomposed as follows:
\begin{align*}
\max_{\{\rho_j\}} \quad & \sum_{j=1}^{J} \pi_j \hat{v}_j(\rho_j), \\
\text{s.t.} \quad & \sum_{j=1}^{J} \pi_j \rho_j \le \frac{B}{T},
\end{align*}
where
\begin{align}
(\text{SP}_j) \quad \hat{v}_j(\rho_j) = \max_{\{p_{j,k}\}} & \sum_{k=1}^{K} p_{j,k} u_{j,k}, \\
\text{s.t.} \quad & \sum_{k=1}^{K} p_{j,k} c_{j,k} \le \rho_j. \label{eq:subproblem} \\
&\sum_k p_{j,k} \le 1\\
&p_{j,k} \in [0,1]
\end{align}

After we introduce the transformations of $\tilde{p}, \tilde{u}$, and $\tilde{c}$, 

For each $j$, sort actions by $\eta_{j,k}=u_{j,k}/c_{j,k}$ in descending order and introduce the transformations of $\tilde{p}, \tilde{u}$, and $\tilde{c}$. 
Let $\tilde{\eta}_{j,k_{j_a}} = \tilde{u}_{j,k_{j_a}} / \tilde{c}_{j,k_{j_a}}$ be the normalized expected reward of virtual action $k_{j_a}$. 
For $a = 1$, using $\frac{u_{j,k_{j_1}}}{c_{j,k_{j_1}}} \ge \frac{u_{j,k_{j_2}}}{c_{j,k_{j_2}}}$, we can show that 
$\tilde{\eta}_{j,k_{j_1}} \ge \tilde{\eta}_{j,k_{j_2}}$. 
For $2 \le a \le K_j - 1$, using 
$
\frac{u_{j,k_{j_a}} - u_{j,k_{j_{a-1}}}}{c_{j,k_{j_a}} - c_{j,k_{j_{a-1}}}} 
\ge 
\frac{u_{j,k_{j_{a+1}}} - u_{j,k_{j_{a-1}}}}{c_{j,k_{j_{a+1}}} - c_{j,k_{j_{a-1}}}},
$
we can show that $\tilde{\eta}_{j,k_{j_a}} \ge \tilde{\eta}_{j,k_{j_{a+1}}}$. 
In other words, we can verify that 
$\tilde{\eta}_{j,k_{j_1}} \ge \tilde{\eta}_{j,k_{j_2}} \ge \cdots \ge \tilde{\eta}_{j,k_{j,K_j}}$.
Thus, without intra-context constraint, the optimal solution 
$\tilde{p}_j^* = [\tilde{p}_{j,k_1}^*, \tilde{p}_{j,k_2}^*, \ldots, \tilde{p}_{j,k_{K_j}}^*]$ 
automatically satisfies 
$\tilde{p}_{j,k_1}^* \ge \tilde{p}_{j,k_2}^* \ge \cdots \ge \tilde{p}_{j,k_{K_j}}^*$. 
Hence, we can remove the intra-context constraint, and thus decouple the probability constraint under a context.

The global LP reduces to a \emph{threshold} solution by greedy algorithm:
sorting all virtual pairs $(j,a)$ by $\tilde\eta$ yields the index $\tilde i(\rho)$ and cumulative masses
$Q_i$ specified in the main text; the solution fully takes indices $\le \tilde i(\rho)$, partially takes
$\tilde i(\rho)+1$, and rejects the rest.

\paragraph{Step 1: Solution formulation of $\mathrm{LP}_{T,B}$.} 

To simplify the LP solution, we first remove context-action pairs that do not contribute to the optimum~\cite{ucb-alp}. Define the reward rate $\eta_{j,k} = u_{j,k} / c_{j,k}$. For context $j$, assume that the remaining candidate set 
$\mathcal{A}_j = \{k_{j_1}, k_{j_2}, \ldots, k_{j_{K_j}}\}$ 
has been sorted in descending order of their reward rates $\eta_{j,k}$.
% $\mathcal{A}_j = \{k_{j,1}, k_{j,2}, \ldots, k_{j,K_j}\}$ 

We then introduce the following transformations:
\[
p_{j,k_{j_a}} =
\begin{cases}
\tilde{p}_{j,k_{j_a}} - \tilde{p}_{j,k_{j_{a+1}}}, & \text{if } 1 \le a \le K_j - 1, \\
\tilde{p}_{j,k_{j_{K_j}}}, & \text{if } a = K_j .
\end{cases}
\]
\[
\tilde{u}_{j,k_{j_a}} =
\begin{cases}
u_{j,k_{j_1}}, & \text{if } a = 1, \\
u_{j,k_{j_a}} - u_{j,k_{j_{a-1}}}, & \text{if } 2 \le a \le K_j ,
\end{cases}
\]
\[
\tilde{c}_{j,k_{j_a}} =
\begin{cases}
c_{j,k_{j_1}}, & \text{if } a = 1, \\
c_{j,k_{j_a}} - c_{j,k_{j_{a-1}}}, & \text{if } 2 \le a \le K_j .
\end{cases}
\]

By formulating the initial LP into the $\mathrm{\widetilde{LP}}_{T,B}$ problem with transformed variables $\tilde{p}$, $\tilde{c}$, and $\tilde{u}$, the intra-context constraint can be omitted, as it is implicitly satisfied after the transformation. This follows from the property of context–action pairs once those that never appear in the optimal solution are removed.
\begin{align*}
(\mathrm{\widetilde{LP}}_{T,B}) \;
&\max_{\tilde p_{j,k}} \quad 
 \sum_{j=1}^{J} \sum_{a=1}^{K_j} 
\pi_j \tilde{p}_{j,k_{j_a}} \, \tilde{u}_{j,k_{j_a}} \\
&\text{s.t.} \quad
\sum_{j=1}^{J} \sum_{a=1}^{K_j} 
\pi_j \tilde{p}_{j,k_{j_a}} \, \tilde{c}_{j,k_{j_a}} 
\le \tfrac{B}{T} \\
&\tilde{p}_{j,k_{j_a}} \in [0,1], 
 \forall j, \text{and } 1 \le a \le K_j 
\end{align*}

After removing the intra-context constraint, all pairs $(j, k_{j_a})$ can be ordered in descending order of their expected reward rate, defined as $\tilde{\eta}_{j,k_{j_a}} = \tilde{u}_{j,k_{j_a}}/\tilde{c}_{j,k_{j_a}}$.
The resulting sequence is denoted by their indices as $\{(j^{(1)}, k^{(1)}), \ldots, (j^{(M)}, k^{(M)})\}$. Let $\rho=B/T$ and $\tilde{i}(\rho) = \max \left\{ i : \sum_{i'=1}^{i} \pi_{j^{(i')}} \tilde{c}_{j^{(i')}, k^{(i')}} \le \rho \right\}$, which represents a threshold index with cumulative cost mass under $\rho$. The optimal solution of $\mathrm{\widetilde{LP}}_{T,B}$ can be expressed as following: 
\[
\tilde{p}_{j^{(i)},k^{(i)}}(\rho) =
\begin{cases}
1, & \text{if } 1 \le i \le \tilde{i}(\rho), \\[4pt]
\dfrac{
\rho - \sum_{i'=1}^{\tilde{i}(\rho)} 
\pi_{j^{(i')}} \tilde{c}_{j^{(i')}, k^{(i')}}}
{\pi_{j^{(\tilde{i}(\rho)+1)}} 
\tilde{c}_{j^{(\tilde{i}(\rho)+1)}, k^{(\tilde{i}(\rho)+1)}}},
& \text{if } i = \tilde{i}(\rho) + 1, \\[6pt]
0, & \text{if } i > \tilde{i}(\rho) + 1 .
\end{cases}
\]

The optimal solution of $\mathrm{LP}_{T,B}$ can be calculated by reversing the transformation.

\section{Proof of Lemma~\ref{lem:sufficient_explore}}
\label{app:sufficient_explore}

\begin{proof}

\textbf{Lowerbound of context-action pair.} We first analyze the number of executions for each context-action pair $(j,k)$ in the exploration stage $\epsilon(T)$. Let $N_j = \sum_{t=1}^{\epsilon(T)}\mathbbm{1}(X_t = j)$ be the number of occurences of context j up to round $\epsilon(T)$. Since each context $X_t$ arrives i.i.d. at each timestep, we can use Hoeffding's Inequality for each context $j$.
\begin{align*}
& \mathbb{P}{\Bigl[ \forall j \in \mathcal{X},  
N_j  \geq (1-\delta)\pi_j \epsilon(T) }\Bigr] \\
&\geq 1 - \sum_{j=1}^J \mathbb{P}\Bigl[ N_j < (1-\delta)\pi_j \epsilon(T) \Bigr] \nonumber \\
&\geq 1 - J e^{-2\delta^2 \epsilon(T)} \nonumber \\
&\geq 1 - J e^{-2\log T} \nonumber \\
&= 1 - J T^{-2}
\end{align*}

From Lemma 1, we have the lower bound $(1 - \delta)\pi_j \epsilon(T) \geq K + \frac{128 K \log T}{L^2}$. 
From the implementation of the exploration stage in Algorithm~\ref{alg:eps-first-alp}, we know that if 
$N_j \geq (1 - \delta)\pi_j \epsilon(T)$, then
\begin{align}
\label{eq:n_min}
N_{j,k}
&\ge \Big\lfloor 1 + \frac{128\,\log T}{L^2} \Big\rfloor \ge \frac{128\,\log T}{L^2}, \qquad \forall k\in\mathcal A. 
\end{align}

Therefore,
\begin{align*}
\mathbb{P}\Bigl[\forall j\in\mathcal X,\,\forall k\in\mathcal A:\;
&N_{j,k} \ge \tfrac{128\,\log T}{L^2}\Bigr]\ge 1 - JT^{-2}
\end{align*}

\textbf{Condition on observing sign flip.} Now, we want to investigate, at the end of exploration stage, the probability of having the order $\bar\xi_{j_1, k_{11}, k_{12}} \geq\bar\xi_{j_2, k_{21}, k_{22}}$  while the ground truth order is $\xi_{j_1, k_{11}, k_{12}} < \xi_{j_2, k_{21}, k_{22}}$. In another word, we want to show the probability of having such flipped order is small.
Define the utility and cost differences
\begin{align*}
\Delta u_i &:= u_{j_i,k_{i1}} - u_{j_i,k_{i2}}, \quad 
\Delta c_i := c_{j_i,k_{i1}} - c_{j_i,k_{i2}}, \quad i = 1,2,
\end{align*}
with empirical counterparts $\Delta\bar{u}_i := \bar{u}_{j_i,k_{i1}} - \bar{u}_{j_i,k_{i2}}$ and $\Delta\bar{c}_i := \bar{c}_{j_i,k_{i1}} - \bar{c}_{j_i,k_{i2}}$. In order to apply the Hoefdding's inequality, define the estimation errors
\[
\delta u_i := \Delta\bar{u}_i - \Delta u_i, \quad \delta c_i := \Delta\bar{c}_i - \Delta c_i.
\]

To avoid working with Hoeffding's inequality on ratios, we apply with the cross-products and define
\begin{align}
Z &:= \Delta u_1 \Delta c_2 - \Delta u_2 \Delta c_1 = (\xi_1 - \xi_2)\Delta c_1 \Delta c_2, \label{eq:Z_def}\\
\bar{Z} &:= \Delta\bar{u}_1 \Delta\bar{c}_2 - \Delta\bar{u}_2 \Delta\bar{c}_1. \label{eq:Zbar_def}
\end{align}

Using the triangle inequality together with $u,\bar u\in[0,1]$ and $c, \bar c \in (0,1]$ gives
\begin{align*}
|\bar Z-Z|
\le &
\bigl(|\delta u_1|+|\delta u_2|\bigr) 
+\bigl(|\delta c_1|+|\delta c_2|\bigr) \\
& +|\delta u_1\delta c_2|+|\delta u_2\delta c_1|.
\end{align*}
If $\xi_{j_1, k_{11}, k_{12}} < \xi_{j_2, k_{21}, k_{22}}$, we have $|Z| = |\xi_{j_2, k_{21}, k_{22}} - \xi_{j_1, k_{11}, k_{12}} ||\Delta c_1||\Delta c_2| \geq \Delta^{\xi}_{min}(\Delta^{c}_{min})^2 = L$. $|\bar Z - Z| <|Z|$ should be a sufficient condition to guarantee no flip between ground truth and estimation. Therefore, we need to bound $|\bar Z - Z|$ with  $|\bar Z - Z| < \Delta^{\xi}_{min}(\Delta^{c}_{min})^2$. 

\textbf{Concentration of empirical estimates.}
For each context-action pair $(j,k)$, let $\bar{u}_{j,k}$ and $\bar{c}_{j,k}$ denote the sample means of utility and cost from $N_{j,k}$ i.i.d.\ observations. Since $u_{j,k} \in [0,1]$ and $c_{j,k} \in [0,1]$ for each observation $s$, Hoeffding's inequality gives
\begin{align}
\mathbb{P}\bigl\{|\bar{u}_{j,k} - u_{j,k}| \geq t\bigr\} &\leq 2\exp(-2N_{j,k}t^2), \label{eq:hoeff_u}\\
\mathbb{P}\bigl\{|\bar{c}_{j,k} - c_{j,k}| \geq t\bigr\} &\leq 2\exp(-2N_{j,k}t^2). \label{eq:hoeff_c}
\end{align}

Set the threshold
\begin{equation}
t := \frac{L}{8\sqrt{2}}.
\label{eq:threshold}
\end{equation}
Using $N_{j,k} \geq \frac{128\log T}{L^2}$ from \eqref{eq:n_min}:
\begin{align*}
2\exp(-2N_{j,k}t^2) 
&\leq 2\exp\left(-2 \cdot \frac{128\log T}{L^2} \cdot \frac{L^2}{128}\right) \\
&= 2\exp(-2\log T) = 2T^{-2}.
\end{align*}

For a ground truth order $\xi_{j_1,k_{11},k_{12}} < \xi_{j_2,k_{21},k_{22}}$, the event $\bar{\xi}_{j_1,k_{11},k_{12}} \geq \bar{\xi}_{j_2,k_{21},k_{22}}$ can only occur if some context $j$ has fewer than $(1-\delta)\pi_j\epsilon(T)$ occurrences or at least one of the $8$ pairs fails the concentration bound. Overall the total failure probability contributed is at most $JT^{-2} + 16T^{-2}$.

\textbf{Upper bound on $|\bar{Z} - Z|$.} If each single estimation error satisfies $|\bar{u}_{j,k} - u_{j,k}| < t$ and $|\bar{c}_{j,k} - c_{j,k}| < t$, we have
\[
|\delta u_i| \leq 2t, \quad |\delta c_i| \leq 2t, \quad i = 1,2.
\]
Since utilities and costs are both in $[0,1]$, the differences satisfy
\[
|\Delta u_i| \leq 1, \quad |\Delta c_i| \leq 1, \quad i = 1,2.
\]

Applying the triangle inequality:
\begin{align}
|\bar{Z} - Z| &\leq |\delta u_1||\Delta c_2| + |\delta u_2||\Delta c_1| + |\Delta u_1||\delta c_2| + |\Delta u_2||\delta c_1| \notag\\
&\quad + |\delta u_1||\delta c_2| + |\delta u_2||\delta c_1| \notag\\
&\leq 2t \cdot 1 + 2t \cdot 1 + 1 \cdot 2t + 1 \cdot 2t + 2t \cdot 2t + 2t \cdot 2t \notag\\
&= 8t + 8t^2 .
\label{eq:Zbar_Z_bound}
\end{align}

Since $|Z|\le 2$ (as $|\Delta u_i|\le 1$ and $|\Delta c_i|\le 1$) and $L \le |Z|$, substituting $t=\frac{L}{8\sqrt{2}}$ from \eqref{eq:threshold} yields:
\begin{align}
|\bar{Z} - Z| &\leq 8 \cdot \frac{L}{8\sqrt{2}} + 8 \cdot\left(\frac{L}{8\sqrt{2}}\right)^2 \notag\\
&= \frac{L}{\sqrt{2}} + \frac{L^2}{16} \notag \\
&\leq \frac{L}{\sqrt{2}} + \frac{2L}{16} \notag \\
&< L.
\label{eq:Zbar_Z_substituted}
\end{align}

Therefore:
\begin{equation}
|\bar{Z} - Z| < L \leq |Z|.
\label{eq:key_inequality}
\end{equation}

For all rankings to be correct, we need each individual estimate to concentrate and enough exploration have been made to each context-action pair:
\begin{align}
\mathbb{P}\{\text{all rankings correct}\} 
&\geq (1 - JT^{-2})(1 - 4JKT^{-2}) \notag\\
&\geq 1 - JT^{-2} - 4JKT^{-2} \notag\\
&= 1 - (4JK + J)T^{-2}.
\label{eq:global_ranking}
\end{align}
\end{proof}

\section{Confidence Level Test}
\label{app:clt}

We present the Confidence Level Test (CLT) for \efalp in Algorithm~\ref{alg:clt}. 
At any time $t$, the algorithm maintains empirical means $(\bar u_{j,k},\bar c_{j,k})$ and execution counts $N_{j,k}$ for each context--action pair.
Using these quantities, we form the empirical marginal reward-cost ratios
$\bar\xi_{j,k_1,k_2} := (\bar u_{j,k_1}-\bar u_{j,k_2})/(\bar c_{j,k_1}-\bar c_{j,k_2})$, and adopt a fixed cost ordering so the cost difference in denominator is positive. For two ratios $\bar\xi_{j_1,k_11,k_12}$ and $\bar\xi_{j_2,k_21,k_22}$ at each time step, CLT runs a confidence test that checks whether the estimated order is the same as the unknown actual order determined with sufficiently high confidence given the current sample counts.
Algorithm~\ref{alg:clt} iterates over all ordered pairs of ratios; if every comparison passes, the policy stops exploration and switches to exploitation using the current empirical statistics; otherwise it continues exploring.

\begin{algorithm}[h]
  \caption{Confidence Level Test for \efalp}
  \label{alg:clt}
  \begin{algorithmic}[1]
    \STATE {\bfseries Input:} confidence target $\delta$; empirical means $\{\bar u_{j,k}, \bar c_{j,k}\}$ and counts $\{N_{j,k}\}$
    \STATE {\bfseries Output:} $\texttt{flagSucc}$
    \STATE $H \leftarrow \log(2/\delta)$; $r_u(n) \leftarrow \sqrt{H/(2n)}$; $r_c(n) \leftarrow \sqrt{H/(2n)}$
    \STATE $\texttt{flagSucc} \leftarrow \texttt{true}$
    \FORALL{$(j_1,k_{11},k_{12},j_2,k_{21},k_{22})$ s.t.\ $k_{11}\neq k_{12}$, $k_{21}\neq k_{22}$, $\bar c_{j_1,k_{11}}>\bar c_{j_1,k_{12}}$, $\bar c_{j_2,k_{21}}>\bar c_{j_2,k_{22}}$}
      \STATE $\Delta\bar u_1 \leftarrow \bar u_{j_1,k_{11}}-\bar u_{j_1,k_{12}}$; $\Delta\bar c_1 \leftarrow \bar c_{j_1,k_{11}}-\bar c_{j_1,k_{12}}$
      \STATE $\Delta\bar u_2 \leftarrow \bar u_{j_2,k_{21}}-\bar u_{j_2,k_{22}}$; $\Delta\bar c_2 \leftarrow \bar c_{j_2,k_{21}}-\bar c_{j_2,k_{22}}$
      \STATE $\delta u_1 \leftarrow r_u(N_{j_1,k_{11}})+r_u(N_{j_1,k_{12}})$; $\delta c_1 \leftarrow r_c(N_{j_1,k_{11}})+r_c(N_{j_1,k_{12}})$
      \STATE $\delta u_2 \leftarrow r_u(N_{j_2,k_{21}})+r_u(N_{j_2,k_{22}})$; $\delta c_2 \leftarrow r_c(N_{j_2,k_{21}})+r_c(N_{j_2,k_{22}})$
      \STATE $\hat Z \leftarrow \Delta\bar u_1 \cdot \Delta\bar c_2 - \Delta\bar u_2 \cdot \Delta\bar c_1$
      \STATE $B \leftarrow \delta u_2\cdot(|\Delta\bar c_1|+\delta c_1) + \delta c_1 + \delta u_1\cdot(|\Delta\bar c_2|+\delta c_2) + \delta c_2 + \delta u_1\delta c_2 + \delta u_2\delta c_1$
      \IF{$|\hat Z| < 2B$}
        \STATE $\texttt{flagSucc} \leftarrow \texttt{false}$
        \RETURN $\texttt{flagSucc}$
      \ENDIF
    \ENDFOR
    \RETURN $\texttt{flagSucc}$
  \end{algorithmic}
\end{algorithm}
CLT is conservative in practice for two reasons. First, LLM pricing tiers place competing models close in cost, so the gaps $\Delta_{\min}^c$ are tiny; the test must accumulate enough samples to distinguish these nearly identical costs with high confidence, requiring far more exploration than is practical. Second, the test must pass for every one of the $O(J^2K^4)$ ratio pairs, so a single hard-to-distinguish pair blocks the switch to exploitation even when all others are well-resolved.
Crucially, CLT is only a \emph{sufficient} condition: the exploitation phase performs well whenever the estimated ordering is approximately correct, even if the test never formally passes. We therefore use a fixed exploration schedule in our experiments.

\section{Experimental Details}
\label{app:exp_details}

\subsection{Dataset}
\label{app:dataset}

\textbf{RouterBench}
RouterBench contains $36,497$ unique queries each paired with responses from 11 different language models of varying sizes and capabilities. In our experiment,we take 10 models from this dataset and their accuracy and cost are shown in Table~\ref{tab:routerbench_performance_cost}. The queries are drawn from eight diverse datasets(MMLU, HellaSwag, GSM8k, ARC Challenge, Winogrande, MBPP, MT-Bench, and RAG).

\begin{table}[h!]
\centering
\caption{Average performance and cost per query across different models on RouterBench.}
\label{tab:routerbench_performance_cost}
\begin{tabular}{lcc}
\hline
Model & Avg Perf. & Avg Cost (\$) \\
\hline
gpt-4-1106-preview          & 0.805 & 7.943e-03 \\
mixtral-8x7b-chat & 0.650 & 4.140e-04 \\
claude-v1                   & 0.648 & 5.870e-03 \\
llama-2-70b-chat       & 0.606 & 1.337e-03 \\
claude-instant-v1           & 0.590 & 1.236e-03 \\
gpt-3.5-turbo-1106          & 0.687 & 7.086e-04 \\
WizardLM-13B-V1.2  & 0.539 & 1.417e-04 \\
code-llama-instruct-34b-chat & 0.504 & 5.499e-04 \\
claude-v2                   & 0.512 & 6.153e-03 \\
mistral-7b-chat   & 0.500 & 1.388e-04 \\
\hline
\end{tabular}
\end{table}

\textbf{SWE-Bench}
For the SWE-Bench dataset, we use its verified subset. We analyze the performance and instance cost of $5$ models downloaded from the benchmark's official repository. Specifically, we include the following models 
\begin{itemize}
    \item \textsc{20250511\_sweagent\_lm\_32b}, referred to as sweagent-lm-32b 
    \item \textsc{20240402\_sweagent\_gpt4\_performance}, referred to as sweagent-gpt-4
    \item \textsc{20240620\_sweagent\_claude3.5sonnet }, referred to as sweagent-claude-3.5-sonnet
    \item \textsc{20240728\_sweagent\_gpt4o}, referred to as sweagent-gpt-4o 
    \item \textsc{20250522\_sweagent\_claude-4-sonnet-20250514}, referred to as sweagent-claude-4-sonnet
\end{itemize}
Those $5$ models' average successful rate and instance cost is shown in Table~\ref{tab:swebench_performance_cost}.

Both datasets are split into training and test sets (50:50) using a fixed random seed for reproducibility, with the same split used across all methods.

\begin{table}[ht]
\centering
\caption{Average performance and cost across different models on SWE-bench.}
\label{tab:swebench_performance_cost}
\begin{tabular}{lcc}
\hline
Model & Avg Perf. & Avg Cost (\$) \\
\hline
sweagent-claude-4-sonnet    & 0.666 & 1.217 \\
sweagent-lm-32b             & 0.402 & 1.039 \\
sweagent-claude-3.5-sonnet  & 0.336 & 1.714 \\
sweagent-gpt-4o             & 0.232 & 2.317 \\
sweagent-gpt-4              & 0.224 & 2.418 \\
\hline
\end{tabular}
\end{table}

\subsection{Baseline Implementation}
\label{app:baseline}

We provide the details of our implementation for each baseline methods:
\begin{itemize}
    % \item \textbf{StaticLP}: Decisions are computed using $( u_{j,k},  c_{j,k})$ at test time without adapting to the remaining budget. Use the same cluster and context-action pair statistics as \alp.
    
    \item  \textbf{FrugalGPT}: A cascading method with a per-query budget constraint.  To control training time latency, we cap the cascade length to be $2$, all other setting uses the default implementation. The cascade parameters are tuned on the training split to meet the target average budget $\rho$ on test set.
    \item \textbf{CascadeRouting}: We implement the method following the original paper, restricted to a cascade length of 2. The estimation strategy varies by dataset: For RouterBench, quality and cost estimators utilize ground-truth statistics with added noise. To ensure a fair comparison with baselines that lack ground-truth access, we adopt the High standard deviation setting as defined in the original work. For SWE-Bench, where we measure instance-level costs, we employ a linear regression model to estimate cost via query embeddings and a logistic regression model to predict quality.
    \item \textbf{SingleBest}: For each cluster $j$, we select the action $k^{\star}(j) = \arg\max_{k:\,c_{j,k}\leq \rho} u_{j,k}$,the action with the highest estimated reward subject to the per-query budget $\rho$. At inference, the router always chooses $k^{\star}(j)$ for queries in cluster $j$. Budget is enforced only through the fixed average constraint $\rho = B/T$, without adaptation over time. Use the same cluster and context-action pair statistics as \alp. Since K-means clustering depends on random centroid initialization, we run five independent seeds ($0$--$4$) and report the mean and standard deviation over the resulting cluster assignments. 
    \item \textbf{MetaLLM}: We adapt the algorithm from the official implementation, using the same cluster and context-action pair statistics as in other baselines. Actions are selected by UCB with a Lagrangian penalty $p \geq 0$ in the reward function to enforce budget adherence. For each budget level, $p$ is tuned via grid search on a held-out validation set (10\% of training data), selecting the value whose total validation cost is closest to the proportionally scaled budget without exceeding it. Results are averaged over 5 random seeds ($0$--$4$), each controlling action tie-breaking and K-means initialization.
    % The router selects action by UCB algorithm, with $p \geq 0$ a Lagrangian cost-penalty parameter built in reward function to control budget adherence. For each budget level, $p$ is chosen via grid search on a held-out validation set (10\% of data), selecting the value whose total validation cost comes closest to the proportionally scaled budget without exceeding it. Results are averaged over 5 random seeds, each controlling the action sampling and the K-means initialisation.

    % \item \textbf{\alp}: 
\end{itemize}

\subsection{Running Environment}
\label{app:exp_setup}
Our FrugalGPT experiments were run as Slurm jobs on a Linux GPU node, requesting one NVIDIA H100, while all other baselines were run on Slurm CPU nodes with 4 CPU tasks and 52GB host memory.

\subsection{Intra-Cluster Variance.}
\label{app:cluster_var}
The intra-cluster homogeneity assumption (Section~\ref{subsec: clustering}) treats
all queries in a cluster as sharing the same expected reward and cost for each model.
Table~\ref{tab:within_cluster_variance} reports the average intra-cluster variance
of reward and normalized cost at $J=16$.
Cost variance is negligible on RouterBench and modest on SWE-Bench, confirming that
cost discretization introduces little approximation error.
% Reward variance reflects genuine per-query utility variation rather than a failure of
% discretization: cluster-level reward estimates correlate with ground-truth statistics
% at Pearson $r=0.997$ (RouterBench) and $r=0.896$ (SWE-Bench), indicating that
% the discretization retains the information needed for routing decisions.

\begin{table}[ht]
\centering
\caption{Average intra-cluster variance of reward and normalized cost at $J=16$.}
\label{tab:within_cluster_variance}
\begin{tabular}{lcc}
\toprule
Dataset      & Avg.\ Reward Variance & Avg.\ Normalized Cost Variance \\
\midrule
RouterBench  & $0.192$               & $0.00049$ \\
SWE-Bench    & $0.199$               & $0.065$   \\
\bottomrule
\end{tabular}
\end{table}

\subsection{\alp Experiment Result on SWE-Bench}
We provide the performance comparison between \efalp and other offline baselines under end-to-end budget in Figure~\ref{fig:epsalp_combined_swebench}.
\label{online-swe}

\begin{figure}
    \centering
    \includegraphics[width=0.6\linewidth]{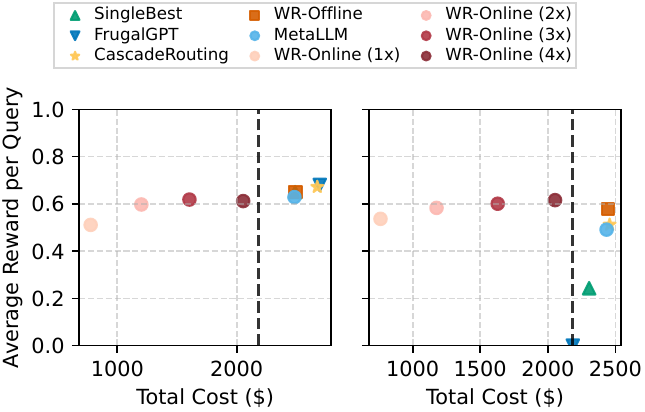}
    \caption{Performance comparison under end-to-end budget constraints on SWE-Bench. Comparison between \efalp and offline baselines. Vertical dashed lines represent the fixed cost of collecting the complete training set for offline methods. \optiroute values are averaged over five runs.}
    \label{fig:epsalp_combined_swebench}
\end{figure}

\subsection{Exploration Estimation Accuracy}
\label{app:estimation_accuracy}

To validate that the exploration phase of \efalp produces reliable statistics for
the subsequent LP routing policy, we compare the empirical estimates $\bar{u}_{j,k}$
and $\bar{c}_{j,k}$ at the end of exploration against ground truth statistics
computed over the full horizon $T$ ($J=16$, exploration length $4\times$ training set).
Table~\ref{tab:estimation_accuracy} reports mean absolute error (MAE) and Pearson
correlation across all context--action pairs $(j,k)$.
Both datasets show high reward correlation, confirming that exploration recovers the
relative model ordering required by the LP.
Cost estimates are nearly exact on RouterBench and accurate on SWE-Bench, where
higher per-instance cost variance makes estimation harder.

\begin{table}[ht]
\centering
\caption{Estimation accuracy of \efalp exploration phase vs.\ ground truth statistics ($J=16$, normalized cost $S_t \in [0,1]$).}
\label{tab:estimation_accuracy}
\begin{tabular}{lcccc}
\toprule
         & \multicolumn{2}{c}{Reward $\bar{u}_{j,k}$} & \multicolumn{2}{c}{Cost $\bar{c}_{j,k}$} \\
\cmidrule(lr){2-3} \cmidrule(lr){4-5}
Dataset      & MAE    & Pearson $r$ & MAE       & Pearson $r$ \\
\midrule
RouterBench  & $0.011$ & $0.997$    & $2.72\times10^{-4}$ & $1.000$ \\
SWE-Bench    & $0.076$ & $0.896$    & $0.041$             & $0.941$ \\
\bottomrule
\end{tabular}
\end{table}

\subsection{Ablation Study}
\label{app:ablation}

Here we present the hyperparameter tuning and ablation study result performed on the RouterBench dataset. All experiment results are conducted on RouterBench train set with cross validation and the validation set ratio is 20\%. Each experiment is runned 5 times and compute the average score. We first examine the sensitivity of our embedding method generated from a commercial API OpenAI’s \textsc{text-embedding-3-small} and an opensource model \textsc{Qwen3-Embedding-0.6B}, Table~\ref{tab:embedding_pca_ablation} shows that the result on validation set over 5 runs is relatively stable, suggesting our method is robust towards the embedding methods and vector dimension choice.

\begin{table}[ht]
\centering
\caption{Average reward across different embedding methods and PCA settings (3285 queries, 5 seeds).}
\label{tab:embedding_pca_ablation}
\begin{tabular}{lcc}
\hline
Embedding & No PCA & PCA (256 dim) \\
\hline
Qwen3-Embedding-0.6B   & 2951.39 {$\pm$ 22.83} & 2951.84 { $\pm$ 21.27} \\
OpenAI-text-embedding-3-small & 2949.65 {$\pm$ 18.74} & 2950.84 { $\pm$ 22.42} \\
\hline
\end{tabular}
\end{table}

We select $J=16$ as the number of clusters for our K-means clustering method. As shown in Table~\ref{tab:kmeans_selection}, all values of $J$ achieve similar average reward. However, $J=16$ attains the highest reward (0.8083) while having the smallest standard deviation (0.0029), indicating the most stable performance across runs.
\begin{table}[ht]
\centering
\caption{Hyperparameter selection for number of clusters $J$ in K-means. Performance is averaged over 5 runs.}
\label{tab:kmeans_selection}
\begin{tabular}{c|c}
\toprule
$J$ & Avg Reward per Query \\
\midrule
8  & $0.8083 \pm 0.0036$ \\
12 & $0.8079 \pm 0.0030$ \\
\textbf{16} & $\mathbf{0.8083 \pm 0.0029}$ \\
20 & $0.8082 \pm 0.0030$ \\
\bottomrule
\end{tabular}
\end{table}

% and DP-Means~\cite{dpmeans}
We also examined other density based clustering methods DBSCAN. Table~\ref{tab:dbscan_selection} shows results when replacing K-means with DBSCAN. Across all $\epsilon$ values we explored, the average reward remains around $0.808$, consistent with our K-means results. This indicates that our method is robust to the choice of clustering algorithm and its hyperparameters.

\begin{table}[ht]
\centering
\caption{Hyperparameter selection for DBSCAN $\epsilon$.}
\label{tab:dbscan_selection}
\begin{tabular}{c|c}
\toprule
$\epsilon$ & Avg Reward per Query \\
\midrule
0.2 & $0.8083 \pm 0.0036$ \\
0.4 & $0.8080 \pm 0.0029$ \\
0.6 & $0.8084 \pm 0.0035$ \\
0.8 & $0.8084 \pm 0.0035$ \\
\bottomrule
\end{tabular}
\end{table}

\subsection{Simulation Experiment}
\label{app:simulation}

\paragraph{Simulation Setup.}
% explain how we make other baseline method to be cost aware (adaptive avg cost)
% explain the difference between efalp-known and efalp-unknown
We consider a synthetic contextual decision-making environment with a finite horizon of $T = 1500$, $J = 11$ discrete contexts, and $K = 11$ available actions. At each round $t$, a context $j \in \{1,\dots,J\}$ is sampled independently from a fixed power-law distribution with parameter $\alpha = 1.0$, inducing a skewed frequency over contexts reflecting reality.

Each context--action pair $(j,k)$ is associated with a ground-truth expected reward $u_{j,k}$ and expected cost $c_{j,k}$. The reward matrix $U \in [0,1]^{J \times K}$ and cost matrix $C \in [0,1]^{J \times K}$ are generated once at initialization by sampling entries independently from a uniform distribution on $[0,1]$, and are held fixed throughout the experiment. These quantities define the expected reward and cost structure of the environment. Observed rewards and costs add noise to its expected value. Budget consumption is tracked using the expected cost $c_{j_t,k_t}$ to ensure that the workload constraint is applies upon the expected cumulative cost.
% At each round, the learner selects an action $k_t$ subject to a global workload-level budget constraint. 

\paragraph{Budget construction.}
To characterize the feasible budget range in the simulated environment, we compute reference workload-level budgets based on the mean cost structure.
Let $c_{j,k}$ denote the mean cost of action $k$ under context $j$, and let $p_j$ be the context arrival distribution.

The minimum per-step expected cost is defined as
$
\underline{c} = \sum_{j} p_j \min_{k} c_{j,k},
$
corresponding to always selecting the cheapest action for each context. Similarly, the maximum per-step expected cost is $\bar c =\sum_{j} p_j \max_{k} c_{j,k}$, which corresponds to always selecting the most expensive action.

For a horizon of $T = 1500$, these values induce a minimum feasible budget $B_{\min} = T \cdot \underline{c}$ and a maximum feasible budget $B_{\max} = T \cdot \overline{c}$. In our simulation, we obtain the mid-level budget $B_{\text{mid}} = \frac{1}{2}(B_{\min} + B_{\max}) = 762.58$,
which induces meaningful trade-offs between reward maximization and budget conservation and is therefore used throughout the simulation experiments.

\paragraph{Making Bandit Algorithms Budget-Aware.}
To enable a fair comparison, all baseline methods are adapted to respect the global budget constraint through an adaptive cost constraint. Specifically, at each round $t$, a baseline policy computes a per-round target cost $b/\tau,$
where $b$ is the remaining budget and $\tau$ is the remaining time step. The policy then restricts its action selection to those actions whose expected cost does not exceed $b/\tau$. If no such action exists, the policy defaults to the lowest-cost action. This mechanism ensures that even baselines that are not explicitly designed for constrained optimization are budget-aware and adapt their behavior over time as budget is consumed.

\paragraph{Compared Methods and Evaluation Protocol.}
We compare \efalp against several baselines modified with the budget aware mechanism described above, including a random policy, $\varepsilon$-greedy policy~\cite{SuttonRL}, and budget-aware LinUCB~\cite{linucb}, all equipped with the adaptive average cost mechanism described above. We additionally evaluate a variants of \efalp when it is provided with the true expected cost value $c_{j,k}$ for all $(j,k)$, we call such method EFALP-known (same as the $\epsilon$-first ALP proposed in~\cite{ucb-alp}). Both variants use identical reward observations and budget accounting, differing only in their access to cost information.
% while EFALP-unknown must estimate expected cost online from noisy cost observations during an initial exploration phase.

An oracle linear-programming policy with full knowledge of $(u_{j,k}, c_{j,k})$ is used to compute the optimal expected reward and serves as the benchmark for regret computation. To ensure a fair comparison, all policies are evaluated on identical context sequences for each random seed. Results are reported with mean and standard deviation over multiple runs.

%%%%%%%%%%%%%%%%%%%%%%%%%%%%%%%%%%%%%%%%%%%%%%%%%%%%%%%%%%%%

%\newpage
%\input{checklist.tex}

\end{document}